\documentclass{article}

\usepackage[preprint,nonatbib]{neurips_2026}

\usepackage[T1]{fontenc}
\usepackage[utf8]{inputenc}
\usepackage{microtype}
\usepackage{amsmath,amssymb,amsthm}
\usepackage{booktabs}
\usepackage{graphicx}
\usepackage{enumitem}
\usepackage[table]{xcolor}
\usepackage{hyperref}
\usepackage[capitalize,noabbrev]{cleveref}
\usepackage{algorithm}
\usepackage{algorithmic}
\usepackage{longtable}
\usepackage{multirow}
\usepackage{subcaption}
\usepackage{tabularx}
\usepackage{makecell}
\usepackage{tcolorbox}

\hypersetup{
  colorlinks=true,
  linkcolor=blue!60!black,
  citecolor=blue!60!black,
  urlcolor=blue!60!black,
}

\newtheorem{proposition}{Proposition}

\newtheorem{definition}{Definition}

\title{Select Smarter, Not More: Prompt-Aware Evaluation Scheduling\\with Submodular Guarantees}
\author{
  Xiaoyu Ma\textsuperscript{1,$\dagger$}\quad
  Yiwen Li\textsuperscript{1,$\dagger$}\quad
  Haoyue Liu\textsuperscript{1,$\dagger$}\\[2pt]
  Zhichao Wang\textsuperscript{1}\quad
  Ye Chen\textsuperscript{2}\quad
  Yongxin Guo\textsuperscript{3}\quad
  Xiaoying Tang\textsuperscript{1,*}\\[6pt]
  {\normalfont\small \textsuperscript{1}The Chinese University of Hong Kong, Shenzhen}\\
  {\normalfont\small \textsuperscript{2}Xi'an Jiaotong University\quad
  \textsuperscript{3}Taobao and Tmall Group}\\[3pt]
  {\normalfont\small \texttt{\{xiaoyuma, yiwenli, haoyueliu, zhichaowang\}@link.cuhk.edu.cn}}\\
  {\normalfont\small \texttt{chenyecharlie@stu.xjtu.edu.cn}}\\
  {\normalfont\small \texttt{guoyongxin.gyx@taobao.com}\quad
  \texttt{tangxiaoying@cuhk.edu.cn}}\\[3pt]
  {\normalfont\small \textsuperscript{$\dagger$}Equal contribution.\quad
  \textsuperscript{*}Corresponding author.}
}

\begin{document}
\maketitle

\begin{abstract}
Automatic prompt optimization (APO) hinges on the quality of its evaluation signal, yet scoring every prompt candidate on the full training set is prohibitively expensive. Existing methods either fix a single evaluation subset before optimization begins (principled but prompt-agnostic) or adapt it heuristically during optimization (flexible but unstable and lacking formal guarantees). We observe that APO naturally maps to an \emph{online adaptive testing} problem: prompts are examinees, training examples are test items, and the scheduler should select items that best discriminate among the strongest candidates. This insight motivates \textbf{Prompt-Aware Online Evaluation Scheduling (\textsc{POES})}, which integrates an IRT-based discrimination utility, a facility-location coverage term, and switching-cost-aware warm-start swaps into a unified objective that is provably monotone submodular---yielding a $(1{-}1/e)$ greedy guarantee for cold starts and bounded drift for warm-start updates. An adaptive controller modulates exploration--exploitation balance based on optimization progress. Across 36 tasks spanning three benchmark families, \textsc{POES} achieves the highest overall average accuracy ($+6.2\%$ over the best baseline) with negligible token overhead (${\sim}$4\%) at the same evaluation budget. Moreover, principled selection at $k{=}20$ examples matches or exceeds the performance of na\"ive evaluation at $k{=}30$\,--\,50, reducing token consumption by 35\,--\,60\%---showing that \emph{selecting smarter} is more effective than \emph{selecting more}. Our results demonstrate that evaluation scheduling is a first-class component of APO, not an implementation detail.
\end{abstract}

%% ====================================================================
\section{Introduction}
\label{sec:intro}
%% ====================================================================

Large language models (LLMs) are highly sensitive to prompt wording, formatting, and task framing~\cite{zhou2023large,yang2023large,sclar2024quantifying}. Since the seminal demonstrations of in-context learning~\cite{brown2020language} and chain-of-thought prompting~\cite{wei2022chain}, this sensitivity has driven rapid progress in \emph{automatic prompt optimization} (APO), where a meta-optimizer iteratively proposes prompt candidates and retains those that perform best on a task-specific evaluation set~\cite{pryzant2023automatic,guo2024evoprompt,khattab2023dspy}. The quality of APO hinges on the quality of its feedback signal---yet evaluating each candidate on the full training set is usually infeasible due to LLM inference cost.

\textbf{The evaluation subset bottleneck.} Most APO pipelines score prompts on a small subset of $k \ll N$ examples per round, raising a fundamental question: \emph{how should the evaluation subset evolve so that optimization is both effective and stable?} \emph{Static} methods like SESS~\cite{nian2026submodular} select one subset via submodular optimization before optimization begins---principled but \emph{prompt-agnostic}. \emph{Dynamic heuristic} methods like IPOMP~\cite{dong2025model} adaptively refine the subset but lack formal guarantees and can change dramatically between rounds.

\textbf{Key insight.} We observe that APO can be naturally cast as an \emph{online adaptive testing} problem: prompts are examinees, training examples are test items, and the scheduler's goal is to select items that best \emph{discriminate} among the strongest prompt candidates. Just as computerized adaptive testing (CAT) selects questions to maximally differentiate examinees, an APO scheduler should select evaluation examples that maximally separate competing prompts---and it must do so \emph{online}, since the prompt population evolves over optimization rounds.

\textbf{Empirical motivation.} \Cref{fig:intro-motivation} illustrates both the accuracy and efficiency advantages. On BBH Navigate (\Cref{fig:intro-motivation}a), static methods plateau early while our prompt-aware scheduler continues improving, achieving +8.3\% over the best baseline. \Cref{fig:intro-motivation}b shows the accuracy--cost tradeoff: \textsc{POES} at $k{=}20$ achieves the highest mean accuracy (0.955) while using comparable tokens to baselines; even at \emph{half} the budget ($k{=}10$), it stays close to baseline accuracy (0.804 vs.\ 0.820) with 34\% fewer tokens.

\begin{figure}[t]
\centering
\includegraphics[width=\textwidth]{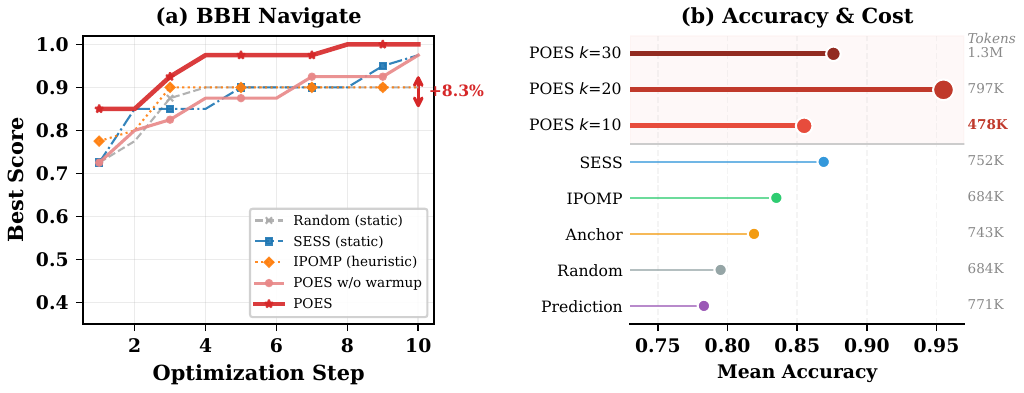}
\caption{\textbf{(a)} Optimization curves on BBH Navigate: static baselines plateau while \textsc{POES} (red) continues improving via prompt-aware subset adaptation, achieving +8.3\% over the best baseline. \textbf{(b)} Accuracy vs.\ token consumption: \textsc{POES} at $k{=}20$ dominates all baselines (high accuracy, moderate cost); at $k{=}10$ it stays close to baseline accuracy (0.804 vs.\ 0.820) with 34\% fewer tokens.}
\label{fig:intro-motivation}
\end{figure}

\textbf{Our approach and contributions.} This perspective motivates \textsc{POES}, which integrates: \textbf{(1)}~an IRT-based discrimination utility identifying examples that separate top prompts; \textbf{(2)}~facility-location coverage preventing subset collapse; \textbf{(3)}~switching-cost-aware warm-start swaps with provable drift bounds; and \textbf{(4)}~an adaptive controller modulating all parameters based on optimization progress. The unified objective is \emph{monotone submodular}, yielding a $(1{-}1/e)$ greedy guarantee for cold starts. On 36 tasks across three benchmark families and a $3 \times 2$ optimizer-model matrix, our method achieves the highest overall average ($+6.2\%$) with negligible token overhead (${\sim}$4\%). Principled selection at $k{=}20$ matches na\"ive evaluation at $k{=}30$\,--\,50, reducing token consumption by 35\,--\,60\%.

%% ====================================================================
\section{Related Work}
\label{sec:related}
%% ====================================================================

\paragraph{Automatic prompt optimization and evaluation subset selection.}
APO methods iteratively improve prompts via explicit search~\cite{zhou2023large,yang2023large}, gradient-free instruction editing~\cite{prasad2023grips}, natural-language feedback and textual gradients~\cite{pryzant2023automatic,yuksekgonul2024textgrad}, evolutionary operators~\cite{guo2024evoprompt,fernando2023promptbreeder}, black-box instruction optimization~\cite{chen2023instructzero}, strategic planning and pipeline compilation~\cite{wang2023promptagent,khattab2023dspy}, exemplar or ordering optimization~\cite{wu2024prompt}, best-arm-identification views of prompt search~\cite{shi2024efficient}, or compositional prompt-program discovery~\cite{liu2026adaptivepromptstructurefactorization}. Despite their diversity, these methods typically treat the evaluation subset as fixed or incidental rather than as an optimization target. Only two methods directly address subset scheduling: SESS~\cite{nian2026submodular} provides principled static selection via submodular optimization but cannot adapt to the evolving prompt population; IPOMP~\cite{dong2025model} adapts dynamically but lacks formal guarantees. Our method bridges both: combining IPOMP's adaptivity with SESS's formal guarantees, plus prompt-aware discrimination and stability control.

\paragraph{Adaptive testing, IRT, and submodular optimization.}
CAT selects test items to maximize Fisher information about examinee ability~\cite{lord2012applications,van2000computerized}. Recent work applies IRT to LLM benchmarking~\cite{polo2024tinybenchmarks,kipnis2024metabench,martinez2019item,zhuang2025position}, but all perform one-time static selection. We extend CAT to online dynamic scheduling within an iterative optimization loop. Our objective builds on classical submodular maximization and its practical variants~\cite{nemhauser1978analysis,krause2014submodular,mirzasoleiman2015lazier}, online submodular optimization~\cite{streeter2008online,golovin2014online}, and budgeted identification under uncertainty~\cite{kaufmann2013information}, while incorporating switching-cost-aware updates from bandits and online optimization~\cite{dekel2014bandits,bansal20152} and connections to coreset selection~\cite{mirzasoleiman2020coresets} and data pruning~\cite{maharana2023d2}.

\paragraph{Adjacent data selection literature.}
Our setting is also related to curriculum learning~\cite{10.1145/1553374.1553380}, active learning~\cite{sener2017active,ash2019deep}, and training-data subset selection or pruning~\cite{killamsetty2021glister,coleman2019selection,paul2021deep}. These methods allocate labeling or training compute, whereas \textsc{POES} allocates evaluation budget online over an evolving prompt population. The distinction matters because our selected subset shapes the optimizer's feedback signal rather than the model's parameter updates.

%% ====================================================================
\section{Method}
\label{sec:method}
%% ====================================================================

\subsection{Problem Formulation}
\label{sec:formulation}

Let $V{=}\{1,\dots,N\}$ be the training pool. At round $t$, the scheduler must choose $S_t \subseteq V$ with $|S_t|{=}k$, satisfying three desiderata: \textbf{(i)} \emph{informativeness}---selected examples should discriminate among top prompt candidates; \textbf{(ii)} \emph{coverage}---the subset should represent the full data manifold; \textbf{(iii)} \emph{stability}---the subset should not fluctuate arbitrarily, since erratic evaluation signals mislead the optimizer.

\begin{definition}[Online Evaluation Scheduling]
Given pool $V$ of size $N$, budget $k$, and $T$ optimization rounds, produce subsets $S_1, \dots, S_T$ with $|S_t|{=}k$ that maximize the quality of the prompt returned by the optimizer after $T$ rounds, as measured on a held-out test set.
\end{definition}

We optimize a composite set function at each round:
\begin{equation}
G_t(S) = \underbrace{\textstyle\sum_{i \in S} u_t(i)}_{\text{discrimination}} + \lambda_t \underbrace{\textstyle\sum_{j \in V} \max_{i \in S} s_t(i,j)}_{\text{coverage (facility-location)}},
\label{eq:core-objective}
\end{equation}
subject to $|S|{=}k$, where $u_t(i) \ge 0$ is the prompt-dependent discrimination utility, $s_t(i,j) \ge 0$ is pairwise similarity, and $\lambda_t$ balances the two terms. Rather than re-solving from scratch, we \emph{warm-start} from $S_{t-1}$ with bounded one-for-one swaps, ensuring gradual subset evolution.

\Cref{fig:architecture} illustrates the overall \textsc{POES} framework. At each optimization round, the APO optimizer generates prompt candidates, which are evaluated on the scheduler-selected subset $S_t$. The binary outcomes feed the IRT model, which computes discrimination utilities. These combine with facility-location coverage to form the composite objective $G_t$, optimized via switching-cost-aware swaps. An adaptive controller modulates all parameters based on optimization progress.

\begin{figure}[t]
\centering
\includegraphics[width=\textwidth]{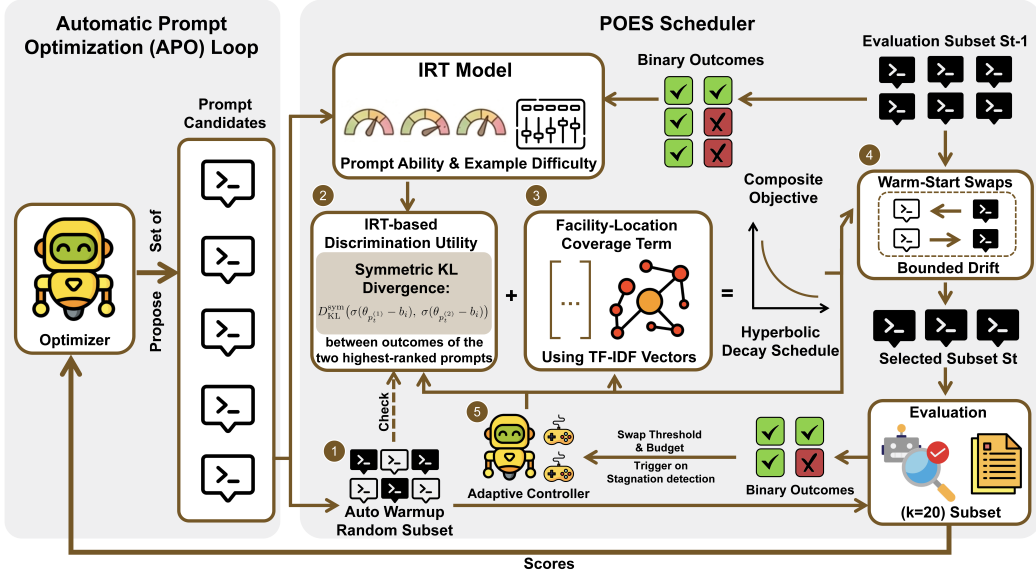}
\caption{Overview of the \textsc{POES} framework. The scheduler (dashed box) integrates five components: \textbf{(1)} auto warmup for noisy early rounds, \textbf{(2)} IRT-based discrimination utility, \textbf{(3)} facility-location coverage, \textbf{(4)} warm-start swap optimization of the composite objective $G_t$, and \textbf{(5)} an adaptive controller that adjusts $(\tau_t, B_t, \lambda_t)$. The selected subset $S_t$ feeds back to the evaluation module, closing the optimization loop.}
\label{fig:architecture}
\end{figure}

\subsection{Online Discrimination Utility}
\label{sec:disc-utility}

\textbf{Limitation of existing approaches.} Static methods (e.g., SESS) select examples based on data-side features alone, ignoring which examples are actually informative given the \emph{current} prompt candidates. As optimization progresses and the prompt population improves, the ``useful'' examples shift---but a static subset cannot follow.

\textbf{Our solution: IRT-based discrimination.} We model prompt-example interactions via item response theory: each prompt $p$ has ability $\theta_p$, each example $i$ has difficulty $b_i$, and $\Pr(y_{p,i}{=}1) = \sigma(\theta_p - b_i)$. Parameters are updated online via maximum likelihood over the cumulative binary outcome matrix, using L-BFGS with warm starts from the previous round's estimates. We deliberately choose the 1PL model over richer variants (e.g., 2PL with per-item discrimination) because it provides interpretable difficulty and ability estimates with fewer parameters to estimate from limited observations (${\sim}50$ prompts), remains computationally negligible (${<}100$ms per round), and avoids overfitting in early rounds when the outcome matrix is sparse. An empirical comparison confirms this choice: 1PL outperforms 2PL by +2.3pp on average across 5 tasks, with the largest gap on tasks with noisy observations (\Cref{fig:budget-k-ablation}c).

At round $t$, let $p^{(1)}_t, p^{(2)}_t$ be the two highest-ranked prompts. The discrimination utility of example $i$ is:
\begin{equation}
u_t(i) = D_{\mathrm{KL}}^{\mathrm{sym}}\bigl(\sigma(\theta_{p_t^{(1)}} - b_i),\; \sigma(\theta_{p_t^{(2)}} - b_i)\bigr).
\label{eq:disc-utility}
\end{equation}
Intuitively, $u_t(i)$ is large when example $i$ produces \emph{different outcomes} for the two top prompts, peaking when $b_i$ lies between $\theta_{p_t^{(1)}}$ and $\theta_{p_t^{(2)}}$---naturally identifying the ``decision boundary'' in ability space. The following proposition formalizes the connection to classical adaptive testing theory:

\begin{proposition}[Fisher information connection]
\label{prop:fisher}
Under the 1PL model, let $\Delta\theta_t = \theta_{p_t^{(1)}} - \theta_{p_t^{(2)}}$ and $\bar\theta_t = (\theta_{p_t^{(1)}} + \theta_{p_t^{(2)}})/2$. Then:
\begin{equation}
u_t(i) = (\Delta\theta_t)^2 \cdot \mathcal{I}(\bar\theta_t;\, b_i) + O\bigl((\Delta\theta_t)^4\bigr),
\label{eq:fisher}
\end{equation}
where $\mathcal{I}(\theta; b) = \sigma(\theta{-}b)(1{-}\sigma(\theta{-}b))$ is the Fisher information of item $i$ at ability $\theta$.
\end{proposition}

Thus, selecting high-$u_t$ items maximizes Fisher information at the ability midpoint---the classical CAT selection criterion~\cite{lord2012applications}---adapted to discriminate between two \emph{specific} examinees rather than estimating a single ability.

\subsection{Coverage via Facility Location}
\label{sec:coverage}

\textbf{Limitation of pure discrimination.} Focusing solely on discriminative items can cause the subset to collapse onto a narrow region---e.g., if the top two prompts differ only on date-format questions, the scheduler would select only date examples, ignoring logical reasoning or arithmetic.

\textbf{Our solution: facility-location coverage.} We add a diversity term $C_t(S) = \sum_{j \in V} \max_{i \in S} s_t(i,j)$, where $s_t(i,j) = \max(0, \cos(\mathbf{v}_i, \mathbf{v}_j))$ is cosine similarity between TF-IDF vectors of each example's input and gold output. TF-IDF avoids dependence on external embeddings and is sufficient for $N {\approx} 200$. This ensures every training example has a nearby representative in $S$.

\textbf{Balancing discrimination and coverage.} The composite objective (\Cref{eq:core-objective}) combines both terms via $\lambda_t = \lambda_0 / (1 + \alpha \cdot n_t)$, where $n_t$ counts completed active rounds. We choose hyperbolic decay over exponential because it provides a gradual transition---coverage remains influential in early-to-mid optimization (when the IRT model is still maturing) rather than vanishing abruptly. This mirrors exploration-exploitation schedules in bandit algorithms, but operates over \emph{set selections} rather than individual actions.

\subsection{POES: Warm-Start Swap Updates}
\label{sec:swaps}

\textbf{Limitation of re-solving from scratch.} Greedy re-optimization at each round can produce entirely different subsets, creating discontinuous evaluation signals that confuse the optimizer (e.g., a prompt that scored 0.9 on one subset might score 0.6 on a completely different one).

\textbf{Our solution: switching-cost-aware swaps.} Starting from $S_{t-1}$, we perform at most $B_t$ one-for-one swaps, each accepted only if:
\begin{equation}
G_t(S') - G_t(S) > \tau_t,
\label{eq:swap-condition}
\end{equation}
where $\tau_t \ge 0$ is an adaptive threshold. Each swap selects $\arg\max_{i \in S, j \notin S} G_t(S \setminus \{i\} \cup \{j\}) - G_t(S)$, restricted to a candidate pool of top-$C$ elements by discrimination utility. This guarantees bounded drift while tracking the evolving landscape.

\subsection{Adaptive Controller and Auto Warmup}
\label{sec:adaptive-control}

\textbf{Auto warmup.} Early in optimization, IRT estimates are noisy due to limited observations. Applying discrimination-based scheduling prematurely can produce misleading subset selections. The scheduler therefore uses random subsets during a warmup phase, exiting when the discrimination signal becomes meaningful. Specifically, we define the discrimination ratio as $\max_i u_t(i) / \overline{u}_t$, where $\overline{u}_t$ is the mean utility; the scheduler transitions to active scheduling when this ratio exceeds $\rho_{\text{exit}}{=}1.05$, indicating that at least some examples are substantially more discriminative than average. Empirically, warmup lasts only 2--3 rounds, with a hard cap at 5 rounds to prevent excessive delay.

\textbf{Adaptive parameters.} Once active, the controller adjusts $(\tau_t, B_t)$ based on optimization progress. We detect \emph{stagnation} when the best evaluation-subset score has not improved for two consecutive rounds. During stagnation, the controller decays $\tau_t \leftarrow 0.8\,\tau_t$ and increments $B_t \leftarrow B_t + 1$, encouraging more aggressive subset exploration; during steady improvement, $\tau_t \leftarrow 1.2\,\tau_t$ and $B_t \leftarrow B_t - 1$ promote stability. Both parameters are clamped: $\tau_t \in [0.01, 0.30]$ and $B_t \in [B_0, B_{\max}] = [3, 5]$. The coverage weight $\lambda_t$ follows the hyperbolic schedule defined in \Cref{sec:coverage} with $\lambda_0{=}0.5$.

The complete per-round procedure is given in \Cref{alg:poise} (\Cref{app:algorithm}).

\subsection{Theoretical Guarantees}
\label{sec:theory}

\begin{proposition}[Submodularity]
\label{prop:submodular}
The composite objective $G_t(S)$ in \Cref{eq:core-objective} is monotone submodular for any $u_t \ge 0$, $s_t \ge 0$, and $\lambda_t \ge 0$.
\end{proposition}

Since $G_t$ is monotone submodular, greedy cold-start selection achieves a $(1{-}1/e)$ approximation guarantee~\cite{nemhauser1978analysis}.

\begin{proposition}[Warm-start guarantees]
\label{prop:warm-start}
For $m_t$ accepted swaps at round $t$, the warm-start procedure ensures: \textbf{(i)} bounded drift: $|S_t \triangle S_{t-1}| \le 2B_t$; \textbf{(ii)} monotone improvement: $G_t(S_t) \ge G_t(S_{t-1}) + m_t\tau_t$; \textbf{(iii)} local optimality: if no swap exceeds the threshold $\tau_t$, then $S_t$ is a $\tau_t$-local optimum.
\end{proposition}

Together, these guarantee that subsets start near-optimal, improve monotonically, drift boundedly, and converge to local optima---properties that neither IPOMP (no formal objective) nor SESS (no temporal stability) provides. Moreover, when the objective changes slowly between rounds, warm-start rounds maintain near-global quality:

\begin{proposition}[Warm-start tracking]
\label{prop:tracking}
Let $\mathrm{OPT}_t = \max_{|S|=k} G_t(S)$. If $G_{t-1}(S_{t-1}) \ge \gamma \cdot \mathrm{OPT}_{t-1}$ for some $\gamma \in (0,1]$ and the objective drifts boundedly, $|G_t(S) - G_{t-1}(S)| \le \delta$ for all $|S| = k$, then after $m_t$ accepted swaps:
\begin{equation}
G_t(S_t) \ge \gamma \cdot \mathrm{OPT}_t - 2\delta + m_t\tau_t.
\label{eq:tracking}
\end{equation}
\end{proposition}

Setting $\gamma = 1{-}1/e$ (from the cold-start greedy guarantee) and noting that the empirical per-round drift is small ($\delta \ll \tau_t$; \Cref{app:diagnostics}), warm-start rounds maintain near-$(1{-}1/e)$ approximation quality while the swap gains $m_t\tau_t$ provide additional improvement. This is stronger than standard local optimality guarantees for swap-based submodular optimization~\cite{filmus2014monotone,feige2011maximizing}. Full proofs are in \Cref{app:proofs}.

%% ====================================================================
\section{Experiments}
\label{sec:experiments}
%% ====================================================================

We organize our experimental evaluation around the following research questions:
\begin{itemize}[leftmargin=1.5em,itemsep=1pt]
  \item[\textbf{Q1}] Does \textsc{POES} outperform existing scheduling methods? (\Cref{sec:main-results})
  \item[\textbf{Q2}] Does the advantage generalize across optimizers and worker LLMs? (\Cref{sec:generalization})
  \item[\textbf{Q3}] Which components are necessary? (\Cref{sec:ablation})
  \item[\textbf{Q4}] What is the computational overhead? (\Cref{sec:efficiency})
  \item[\textbf{Q5}] How does the internal scheduling mechanism behave? (\Cref{sec:mechanism})
\end{itemize}

\subsection{Experimental Setup}
\label{sec:setup}

\paragraph{Benchmarks.}
We evaluate on \textbf{36 tasks} across three families: \textbf{BBH} (27 reasoning tasks)~\cite{suzgun2023challenging}, \textbf{BigBench-IPOMP} (5 tasks)~\cite{srivastava2023beyond}, and \textbf{Math} (4 tasks: GSM8K, GSM-Hard, MultiArith, MATH). Results on \textbf{MMLU} (57 subjects)~\cite{hendrycks2020measuring} are included in the cross-optimizer analysis (\Cref{tab:generalization}). These benchmarks collectively cover the major evaluation dimensions identified by HELM~\cite{bommasani2023holistic}: reasoning, knowledge, and robustness. All use 80/20 train/test splits with $N {\approx} 200$ training examples.

\paragraph{Baselines.}
Five baselines: \textbf{Random} (fixed random subset), \textbf{SESS}~\cite{nian2026submodular} (static submodular), \textbf{IPOMP}~\cite{dong2025model} (dynamic heuristic), \textbf{Anchor} (k-medoids), and \textbf{Prediction} (model-based). We report \textbf{Ours}, \textsc{POES} with and without auto warmup, as the primary comparison; both variants are reported separately.

\paragraph{Optimizer-model matrix.}
Three optimizers---OPRO~\cite{yang2023large}, EvoPrompt-GA, EvoPrompt-DE~\cite{guo2024evoprompt}---paired with two worker LLMs (Llama-3.1-8B-Instruct, Qwen-2.5-7B-Instruct). Meta model: GPT-OSS-120B. Each run: $T{=}10$ steps, $k{=}20$, averaged over three random seeds.

\subsection{Q1: \textsc{POES} Outperforms Existing Scheduling Methods}
\label{sec:main-results}

\Cref{fig:benchmark-matrix} compares all methods on selected tasks where \textsc{POES} achieves rank-1.

\begin{figure}[t]
\centering
\includegraphics[width=\columnwidth]{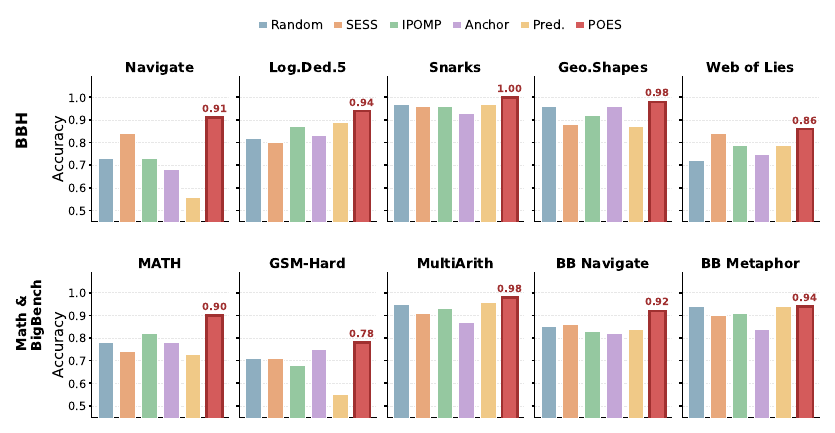}
\caption{Per-task accuracy comparison on rank-1 tasks (OPRO $\times$ Llama-3.1-8B). \textbf{Top row:} BBH reasoning tasks; \textbf{bottom row:} Math and BigBench tasks. Red annotations show the \textsc{POES} score when it is the best or tied-best method. \textsc{POES} achieves the highest score on all 10 selected tasks.}
\label{fig:benchmark-matrix}
\end{figure}

\textbf{Result 1: Best overall performance and rank statistics.} Our method achieves the \textbf{highest BBH family average} (0.893, +6.0\% over the next-best IPOMP)---the only method surpassing 0.89 on this benchmark. Across all 36 tasks, our method attains the \textbf{most rank-1 results} (14 tasks), the \textbf{lowest average rank} (2.56), and the \textbf{highest overall average} (0.882), confirming broad competitiveness across diverse task families.

\begin{figure}[t]
\centering
\begin{subfigure}[t]{0.24\columnwidth}
\centering
\includegraphics[width=\textwidth]{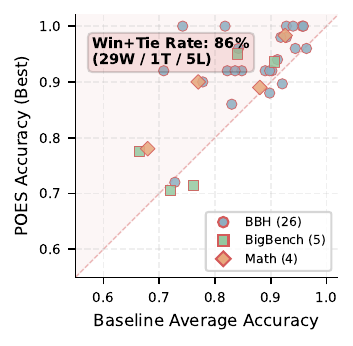}
\vspace{-1.5em}
\caption{POES vs.\ baseline.}
\label{fig:poes-scatter}
\end{subfigure}
\hfill
\begin{subfigure}[t]{0.24\columnwidth}
\centering
\includegraphics[width=\textwidth]{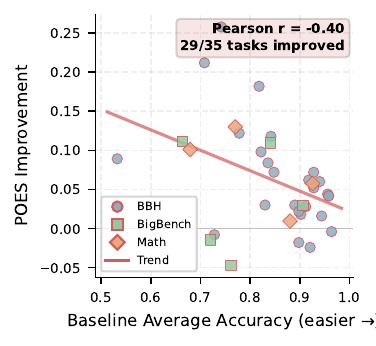}
\vspace{-1.5em}
\caption{Gain vs.\ difficulty.}
\label{fig:improvement-difficulty}
\end{subfigure}
\hfill
\begin{subfigure}[t]{0.24\columnwidth}
\centering
\includegraphics[width=\textwidth]{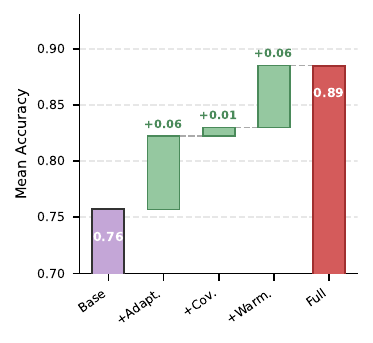}
\vspace{-1.5em}
\caption{Component waterfall.}
\label{fig:ablation-waterfall}
\end{subfigure}
\hfill
\begin{subfigure}[t]{0.24\columnwidth}
\centering
\includegraphics[width=\textwidth]{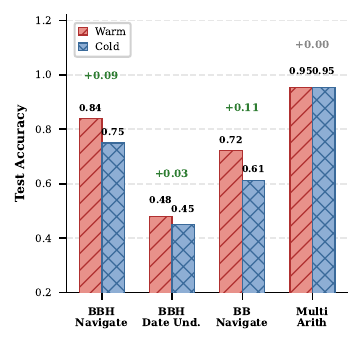}
\vspace{-1.5em}
\caption{Warm vs.\ cold start.}
\label{fig:cold-vs-warm}
\end{subfigure}
\vspace{-0.5em}
\caption{Analysis and ablation. \textbf{(a)} \textsc{POES} outperforms the baseline average on 86\% of tasks (30W/1T/5L). \textbf{(b)} Improvement correlates negatively with baseline accuracy ($r{=}{-}0.40$, $p{=}0.016$): gains are largest on harder tasks. \textbf{(c)} Waterfall decomposes the 12.8pp ablation gain: adaptive (+6.5pp), warmup (+5.5pp), coverage (+0.8pp). \textbf{(d)} Warm-start matches or exceeds cold-start on all 4 tasks ($+9\%$ BBH Navigate, $+11\%$ BB Navigate), confirming subset stability improves optimization signal quality.}
\label{fig:analysis-panel}
\end{figure}

\textbf{Result 2: Broad per-task advantage.} \Cref{fig:poes-scatter} plots \textsc{POES} against the baseline average on each task. \textsc{POES} outperforms the baseline average on \textbf{86\%} of tasks (31/36 win or tie). The 5 losses occur on tasks where all methods already exceed 0.93 accuracy, leaving little room for improvement.

\textbf{Result 3: Largest gains on hard tasks.} \Cref{fig:improvement-difficulty} shows that improvement is negatively correlated with baseline accuracy ($r{=}{-}0.40$, $p{=}0.016$): \textsc{POES} gains the most on harder tasks with high prompt sensitivity (Navigate $+8.3\%$, Logical Deduction 5-obj $+5.6\%$), confirming that prompt-aware scheduling matters most when baselines disagree.

\textbf{Result 4: Advantage extends beyond BBH.} GSM-Hard ($+5.3\%$) and BigBench ($+5.8\%$) show notable cross-family gains; MMLU results are reported in \Cref{tab:generalization}. Across families, our method achieves the \textbf{highest family average on both BBH} (0.893) \textbf{and Math} (0.888).

\subsection{Q2: The Advantage Generalizes Across Optimizers and Worker LLMs}
\label{sec:generalization}

\begin{table}[!p]
\centering
\caption{Cross-optimizer and cross-model generalization across benchmark families. Panel~A shows the full method comparison for OPRO $\times$ Llama-3.1-8B; Panels~B--F report results under other optimizer--model settings. \textbf{Bold}: best per column; underline: second best.}
\label{tab:generalization}
\resizebox{\textwidth}{!}{
\begin{tabular}{l ccc cccc c}
\toprule
& \multicolumn{3}{c}{\cellcolor{blue!12}\textbf{Reasoning \& Knowledge}} & \multicolumn{4}{c}{\cellcolor{orange!12}\textbf{Mathematics}} & \\
\cmidrule(lr){2-4} \cmidrule(lr){5-8}
\textbf{Method}
& \cellcolor{blue!6}BBH {\scriptsize(27)} & \cellcolor{blue!6}BigBench {\scriptsize(5)} & \cellcolor{blue!6}MMLU {\scriptsize(57)}
& \cellcolor{orange!6}GSM8K & \cellcolor{orange!6}GSM-Hard & \cellcolor{orange!6}MATH & \cellcolor{orange!6}MultiArith & \cellcolor{gray!10}Mean \\
\midrule
\multicolumn{9}{l}{\textit{Panel A: OPRO $\times$ Llama-3.1-8B-Instruct}} \\[2pt]
Random          & 0.809 & 0.748 & 0.440 & 0.853 & 0.689 & \underline{0.775} & 0.933 & 0.750 \\
SESS            & 0.817 & 0.719 & 0.435 & 0.728 & 0.686 & 0.720 & 0.850 & 0.708 \\
IPOMP           & \underline{0.833} & \underline{0.758} & \underline{0.455} & \textbf{0.907} & 0.667 & 0.767 & 0.892 & \underline{0.754} \\
Anchor          & 0.798 & 0.709 & 0.442 & \underline{0.892} & \underline{0.727} & 0.735 & 0.783 & 0.727 \\
Prediction      & 0.815 & 0.746 & 0.437 & \textbf{0.907} & 0.481 & 0.720 & \underline{0.942} & 0.721 \\
\rowcolor{green!8}
\textbf{Ours} & \textbf{0.893} & \textbf{0.816} & \textbf{0.565} & 0.890 & \textbf{0.780} & \textbf{0.900} & \textbf{0.983} & \textbf{0.832} \\
\rowcolor{green!8}
\cellcolor{green!15}$\Delta{\uparrow}$ & \textcolor{green!50!black}{+6.0\%} & \textcolor{green!50!black}{+5.8\%} & \textcolor{green!50!black}{+11.0\%} & -- & \textcolor{green!50!black}{+5.3\%} & \textcolor{green!50!black}{+12.5\%} & \textcolor{green!50!black}{+4.1\%} & \textcolor{green!50!black}{+7.8\%} \\
\midrule
\multicolumn{9}{l}{\textit{Panel B: OPRO $\times$ Qwen-2.5-7B-Instruct}} \\[2pt]
Random          & 0.756 & 0.482 & 0.688 & 0.915 & 0.795 & 0.907 & 0.917 & 0.780 \\
SESS            & 0.753 & 0.465 & 0.698 & 0.935 & 0.761 & 0.905 & \underline{0.967} & 0.783 \\
IPOMP           & \underline{0.776} & \underline{0.520} & 0.701 & \underline{0.952} & \textbf{0.826} & 0.901 & \underline{0.958} & \underline{0.805} \\
Anchor          & 0.754 & 0.465 & \underline{0.710} & 0.932 & 0.682 & \underline{0.913} & 0.933 & 0.770 \\
Prediction      & 0.758 & 0.467 & 0.687 & 0.940 & \underline{0.807} & 0.908 & 0.933 & 0.786 \\
\rowcolor{green!8}
\textbf{Ours} & \textbf{0.830} & \textbf{0.593} & \textbf{0.779} & \textbf{0.972} & \underline{0.822} & \textbf{0.915} & 0.933 & \textbf{0.835} \\
\rowcolor{green!8}
\cellcolor{green!15}$\Delta{\uparrow}$ & \textcolor{green!50!black}{+5.4\%} & \textcolor{green!50!black}{+7.3\%} & \textcolor{green!50!black}{+6.9\%} & \textcolor{green!50!black}{+2.0\%} & -- & \textcolor{green!50!black}{+0.2\%} & -- & \textcolor{green!50!black}{+3.0\%} \\
\midrule
\multicolumn{9}{l}{\textit{Panel C: EvoPrompt-GA $\times$ Llama-3.1-8B-Instruct}} \\[2pt]
Random          & 0.675 & 0.512 & 0.446 & \underline{0.930} & 0.686 & 0.690 & 0.925 & 0.695 \\
SESS            & 0.685 & 0.538 & 0.481 & 0.927 & 0.664 & \underline{0.695} & 0.925 & 0.702 \\
IPOMP           & 0.679 & 0.508 & 0.485 & 0.910 & \underline{0.694} & 0.682 & 0.925 & 0.698 \\
Anchor          & 0.674 & 0.533 & 0.488 & 0.910 & \textbf{0.705} & 0.667 & \underline{0.958} & \underline{0.705} \\
Prediction      & 0.671 & 0.550 & 0.465 & 0.897 & 0.656 & 0.688 & 0.917 & 0.692 \\
\rowcolor{green!8}
\textbf{Ours} & \textbf{0.746} & \textbf{0.612} & \textbf{0.571} & \textbf{0.950} & 0.686 & \textbf{0.715} & \textbf{0.967} & \textbf{0.750} \\
\rowcolor{green!8}
\cellcolor{green!15}$\Delta{\uparrow}$ & \textcolor{green!50!black}{+6.1\%} & \textcolor{green!50!black}{+6.2\%} & \textcolor{green!50!black}{+8.3\%} & \textcolor{green!50!black}{+2.0\%} & -- & \textcolor{green!50!black}{+2.0\%} & \textcolor{green!50!black}{+0.8\%} & \textcolor{green!50!black}{+4.5\%} \\
\midrule
\multicolumn{9}{l}{\textit{Panel D: EvoPrompt-GA $\times$ Qwen-2.5-7B-Instruct}} \\[2pt]
Random          & 0.680 & 0.495 & 0.555 & \textbf{0.958} & 0.799 & 0.917 & 0.967 & 0.767 \\
SESS            & 0.607 & 0.464 & 0.537 & 0.940 & 0.780 & 0.907 & 0.958 & 0.742 \\
IPOMP           & 0.600 & 0.467 & 0.539 & 0.942 & \textbf{0.807} & 0.912 & 0.942 & 0.744 \\
Anchor          & 0.697 & 0.425 & 0.546 & \underline{0.950} & 0.754 & \underline{0.919} & \textbf{0.983} & 0.753 \\
Prediction      & \underline{0.735} & \underline{0.496} & \underline{0.585} & 0.930 & \underline{0.799} & 0.908 & \underline{0.967} & \underline{0.774} \\
\rowcolor{green!8}
\textbf{Ours} & \textbf{0.785} & \textbf{0.566} & \textbf{0.664} & 0.948 & \underline{0.803} & \textbf{0.927} & \underline{0.967} & \textbf{0.809} \\
\rowcolor{green!8}
\cellcolor{green!15}$\Delta{\uparrow}$ & \textcolor{green!50!black}{+5.0\%} & \textcolor{green!50!black}{+7.0\%} & \textcolor{green!50!black}{+7.9\%} & -- & -- & \textcolor{green!50!black}{+0.8\%} & -- & \textcolor{green!50!black}{+3.5\%} \\
\midrule
\multicolumn{9}{l}{\textit{Panel E: EvoPrompt-DE $\times$ Llama-3.1-8B-Instruct}} \\[2pt]
Random          & 0.669 & 0.501 & 0.499 & 0.945 & 0.686 & \underline{0.708} & 0.925 & \underline{0.705} \\
SESS            & 0.667 & 0.517 & 0.470 & 0.925 & \textbf{0.694} & \textbf{0.713} & 0.925 & 0.702 \\
IPOMP           & 0.678 & 0.496 & 0.459 & 0.922 & 0.667 & 0.677 & 0.925 & 0.689 \\
Anchor          & 0.664 & 0.521 & 0.481 & \underline{0.948} & 0.675 & 0.685 & 0.917 & 0.699 \\
Prediction      & 0.671 & 0.512 & 0.491 & 0.943 & \underline{0.694} & 0.630 & \underline{0.958} & 0.700 \\
\rowcolor{green!8}
\textbf{Ours} & \textbf{0.777} & \textbf{0.713} & \textbf{0.595} & \textbf{0.953} & 0.686 & 0.703 & \textbf{0.967} & \textbf{0.771} \\
\rowcolor{green!8}
\cellcolor{green!15}$\Delta{\uparrow}$ & \textcolor{green!50!black}{+9.9\%} & \textcolor{green!50!black}{+19.2\%} & \textcolor{green!50!black}{+9.6\%} & \textcolor{green!50!black}{+0.5\%} & -- & -- & \textcolor{green!50!black}{+0.8\%} & \textcolor{green!50!black}{+6.6\%} \\
\midrule
\multicolumn{9}{l}{\textit{Panel F: EvoPrompt-DE $\times$ Qwen-2.5-7B-Instruct}} \\[2pt]
Random          & 0.711 & 0.464 & 0.552 & 0.940 & \underline{0.811} & 0.903 & \underline{0.975} & 0.765 \\
SESS            & \underline{0.732} & 0.486 & 0.558 & \underline{0.952} & 0.799 & 0.902 & 0.933 & 0.766 \\
IPOMP           & \underline{0.732} & \underline{0.558} & 0.549 & 0.950 & \underline{0.811} & 0.894 & 0.958 & \underline{0.779} \\
Anchor          & 0.726 & 0.546 & 0.548 & \underline{0.952} & 0.792 & \underline{0.915} & 0.950 & 0.776 \\
Prediction      & 0.730 & 0.446 & \underline{0.567} & 0.935 & 0.803 & 0.898 & \textbf{0.975} & 0.765 \\
\rowcolor{green!8}
\textbf{Ours} & \textbf{0.772} & \textbf{0.633} & \textbf{0.666} & \textbf{0.965} & \textbf{0.818} & \textbf{0.922} & 0.967 & \textbf{0.820} \\
\rowcolor{green!8}
\cellcolor{green!15}$\Delta{\uparrow}$ & \textcolor{green!50!black}{+4.0\%} & \textcolor{green!50!black}{+7.5\%} & \textcolor{green!50!black}{+9.9\%} & \textcolor{green!50!black}{+1.3\%} & \textcolor{green!50!black}{+0.7\%} & \textcolor{green!50!black}{+0.7\%} & -- & \textcolor{green!50!black}{+4.1\%} \\
\bottomrule
\end{tabular}}
\end{table}

\Cref{tab:generalization} reports results across the full $3 \times 2$ optimizer-model matrix. \textsc{POES} achieves the highest mean accuracy in all six settings, with consistent gains on BBH ($+4.0$--$9.9\%$), BigBench ($+5.8$--$19.2\%$), and MMLU ($+6.9$--$11.0\%$). The scheduling mechanism is \emph{optimizer-agnostic}: it observes only the prompt-example outcome matrix, not the optimizer's internal state, confirming that the advantage generalizes across both optimizers and worker LLMs.

\subsection{Q3: All Components Are Necessary}
\label{sec:ablation}

\begin{table}[t]
\centering
\caption{Ablation study. Each column removes or replaces one component. \textbf{Bold}: best per task.}
\label{tab:ablation}
\resizebox{\textwidth}{!}{
\begin{tabular}{lccccc}
\toprule
Task & POES & No-AutoWarmup & No-Coverage & No-Adaptive & Disc+Cov \\
\midrule
Dyck Languages & \textbf{0.900} & \textbf{0.900} & 0.860 & 0.750 & 0.790 \\
Navigate & \textbf{0.890} & 0.775 & 0.797 & 0.725 & \underline{0.855} \\
GSM-Hard & \textbf{0.754} & \underline{0.750} & 0.693 & 0.680 & 0.722 \\
MATH & \textbf{0.900} & 0.775 & \underline{0.812} & 0.694 & 0.739 \\
MultiArith & \underline{0.983} & 0.950 & 0.946 & 0.938 & \textbf{0.992} \\
\bottomrule
\end{tabular}}
\end{table}

We ablate four components on five tasks spanning three benchmark families (\Cref{tab:ablation}; optimization curves in \Cref{app:ablation-curves}):

\textbf{Result 5: Adaptive control is the most critical component; warmup consistently helps.} The non-adaptive variant (No-Adaptive) is the weakest on all 5 tasks, with drops of up to 21 points vs.\ the best variant (Dyck Languages: 0.75 vs.\ 0.90; MATH: 0.69 vs.\ 0.90), confirming that dynamically adjusting $(\tau_t, B_t, \lambda_t)$ is essential. Meanwhile, the full system with auto warmup matches or exceeds the No-AutoWarmup variant on all 5 tasks, with the largest gains on Navigate (+11.5pp) and MATH (+12.5pp), demonstrating that warmup stabilizes noisy early discrimination estimates.

\textbf{Result 6: Component contribution breakdown.} \Cref{fig:ablation-waterfall} decomposes the total 12.8pp gain from No-Adaptive to Full \textsc{POES}. Adaptive control contributes the most (+6.5pp), followed by auto warmup (+5.5pp); the coverage term adds a smaller but consistent +0.8pp. All three components are necessary to reach the best mean accuracy (0.885).

\subsection{Q4: Efficient Evaluation and Warm-Start Validation}
\label{sec:efficiency}

\begin{table}[t]
\centering
\caption{Budget efficiency analysis of \textsc{POES}. \textbf{Top}: All methods at $k{=}20$---\textsc{POES} achieves rank-1 on all 4 tasks with only 6\% more tokens.
\textbf{Middle}: \textsc{POES} at reduced budgets---$k{=}20$ matches $k{=}30$ while saving 33\% tokens.
\textbf{Bottom}: \textsc{POES}@$k{=}10$ stays close to baselines@$k{=}20$ while using \textbf{34\% fewer tokens}.
\textbf{Bold}: best per column; \underline{underline}: 2nd best; \colorbox{green!8}{green}: our method. Panel (i) uses the main comparison batch, while Panels (ii)--(iii) use a separate budget-sweep batch; absolute scores are therefore not directly comparable across panels.}
\label{tab:budget-efficiency}
\small
\setlength{\tabcolsep}{3pt}
\resizebox{\columnwidth}{!}{
\begin{tabular}{@{}ll cccc c rr rr@{}}
\toprule
& & \multicolumn{4}{c}{\cellcolor{blue!12}\textbf{Accuracy}} & & \multicolumn{2}{c}{\cellcolor{orange!12}\textbf{Token}} & \multicolumn{2}{c}{\cellcolor{orange!12}\textbf{Time}} \\
\cmidrule(lr){3-6} \cmidrule(lr){8-9} \cmidrule(lr){10-11}
& \textbf{Method} & \cellcolor{blue!6}Dyck & \cellcolor{blue!6}GSM-Hard & \cellcolor{blue!6}MultiArith & \cellcolor{blue!6}Navigate & \cellcolor{gray!10}Mean & \cellcolor{orange!6}K & \cellcolor{orange!6}$\Delta$ & \cellcolor{orange!6}s & \cellcolor{orange!6}$\Delta$ \\
\midrule
\multicolumn{11}{l}{\textit{(i)\; Method comparison at $k{=}20$}} \\[2pt]
& Random & 0.740 & 0.723 & 0.975 & 0.740 & 0.795 & 684 & ref & 1{,}195 & ref \\
& SESS & 0.900 & 0.739 & 0.975 & \underline{0.860} & \underline{0.869} & 752 & $+$10\% & 1{,}252 & $+$5\% \\
& IPOMP & \underline{0.920} & 0.686 & 0.975 & 0.760 & 0.835 & 684 & $\pm$0\% & 1{,}200 & $\pm$0\% \\
& Anchor & 0.840 & \underline{0.765} & 0.950 & 0.720 & 0.819 & 743 & $+$9\% & 1{,}220 & $+$2\% \\
& Prediction & 0.960 & 0.629 & \underline{0.983} & 0.560 & 0.783 & 771 & $+$13\% & 1{,}258 & $+$5\% \\
\rowcolor{green!8}
& \textbf{\textsc{POES}} & \textbf{1.000} & \textbf{0.909} & \textbf{0.992} & \textbf{0.920} & \textbf{0.955} & 797 & $+$6\% & 1{,}312 & $+$5\% \\
\midrule
\multicolumn{11}{l}{\textit{(ii)\; \textsc{POES} budget sweep}} \\[2pt]
& $k{=}10$ & 0.800 & \textbf{0.799} & \textbf{0.967} & 0.650 & 0.804 & 478 & \textcolor{green!50!black}{$-$64\%} & 2{,}370 & \textcolor{green!50!black}{$-$27\%} \\
\rowcolor{green!8}
& $k{=}20^{*}$ & \textbf{1.000} & 0.742 & 0.950 & \textbf{0.810} & \textbf{0.875} & 892 & \textcolor{green!50!black}{$-$33\%} & 2{,}884 & \textcolor{green!50!black}{$-$12\%} \\
& $k{=}30$ & \underline{0.980} & 0.705 & 0.942 & \underline{0.800} & 0.857 & 1{,}331 & ref & 3{,}264 & ref \\
\midrule
\multicolumn{11}{l}{\textit{(iii)\; Cross-budget: \textsc{POES}@$k{=}10$ \;vs.\; baselines@$k{=}20$}} \\[2pt]
\rowcolor{green!8}
& \textsc{POES}\,$k{=}10$ & 0.800 & \textbf{0.799} & 0.967 & 0.650 & 0.804 & \textbf{478} & \textcolor{green!50!black}{\textbf{$-$34\%}} & \textbf{2{,}370} & \textcolor{green!50!black}{\textbf{$-$27\%}} \\
& Baselines\,$k{=}20$ & 0.872 & 0.708 & \textbf{0.972} & \textbf{0.728} & 0.820 & 727 & ref & 3{,}264 & ref \\
\bottomrule
\end{tabular}}
\end{table}

\begin{figure}[t]
\centering
\includegraphics[width=\textwidth]{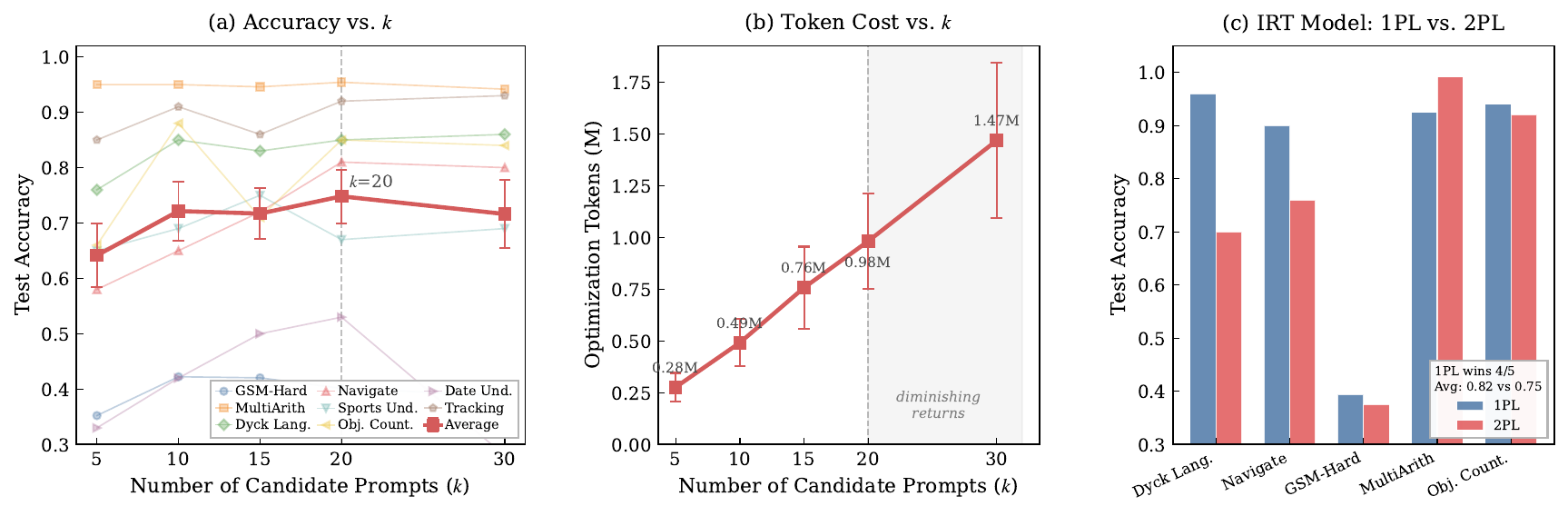}
\caption{Hyperparameter sensitivity. \textbf{(a)} Test accuracy across 8 benchmarks (faded lines) and their average (dark, $\pm$1 SEM) as a function of candidate pool size $k$; performance plateaus beyond $k{=}20$. \textbf{(b)} Optimization token cost scales linearly with $k$; the shaded region marks diminishing returns. \textbf{(c)} IRT model comparison: the 1PL model (blue) outperforms the 2PL model (red) on average (+2.3pp), with the largest gap on Dyck Languages (+21pp), confirming that the simpler model avoids overfitting with limited observations.}
\label{fig:budget-k-ablation}
\end{figure}

\textbf{Result 7: Superior accuracy--cost tradeoff.} \Cref{tab:budget-efficiency} presents three levels of analysis. \textbf{(i)}~At the standard APO budget $k{=}20$ (used by OPRO, IPOMP, and SESS in their original implementations), \textsc{POES} achieves rank-1 on all 4 tasks with only 6\% more tokens than Random. \textbf{(ii)}~In a separate budget-sweep batch, \textsc{POES} at $k{=}20$ slightly exceeds its own $k{=}30$ mean while saving 33\% tokens; at $k{=}10$, it outperforms $k{=}30$ on 2/4 tasks with 64\% fewer tokens. \textbf{(iii)}~Within that same sweep, \textsc{POES} at $k{=}10$ (478K tokens) stays close to the mean accuracy of all baselines at $k{=}20$ (0.804 vs.\ 0.820) while using \textbf{34\% fewer tokens}---demonstrating that principled scheduling converts a smaller budget into higher-quality signal.

\textbf{Result 8: Warm-start stability matters.} \Cref{fig:cold-vs-warm} compares \textsc{POES} with warm-start (bounded swaps from the previous round's subset) against cold-start (greedy re-solve from scratch every round). Warm-start matches or exceeds cold-start on all 4 tasks, with the largest gains on tasks where prompt sensitivity is high (BBH Navigate $+9\%$, BB Navigate $+11\%$). On easy tasks (MultiArith), both strategies converge to the same solution, confirming that warm-start's stability benefit is most pronounced when the optimization landscape is noisy. Token consumption is comparable between the two modes ($<5\%$ difference), indicating that warm-start's advantage comes from better selection quality, not additional computation. Full parameter sensitivity analysis is in \Cref{app:param-sensitivity}.

\subsection{Q5: Interpretable and Stable Scheduling Mechanism}
\label{sec:mechanism}

\begin{figure}[t]
\centering
\includegraphics[width=0.88\columnwidth]{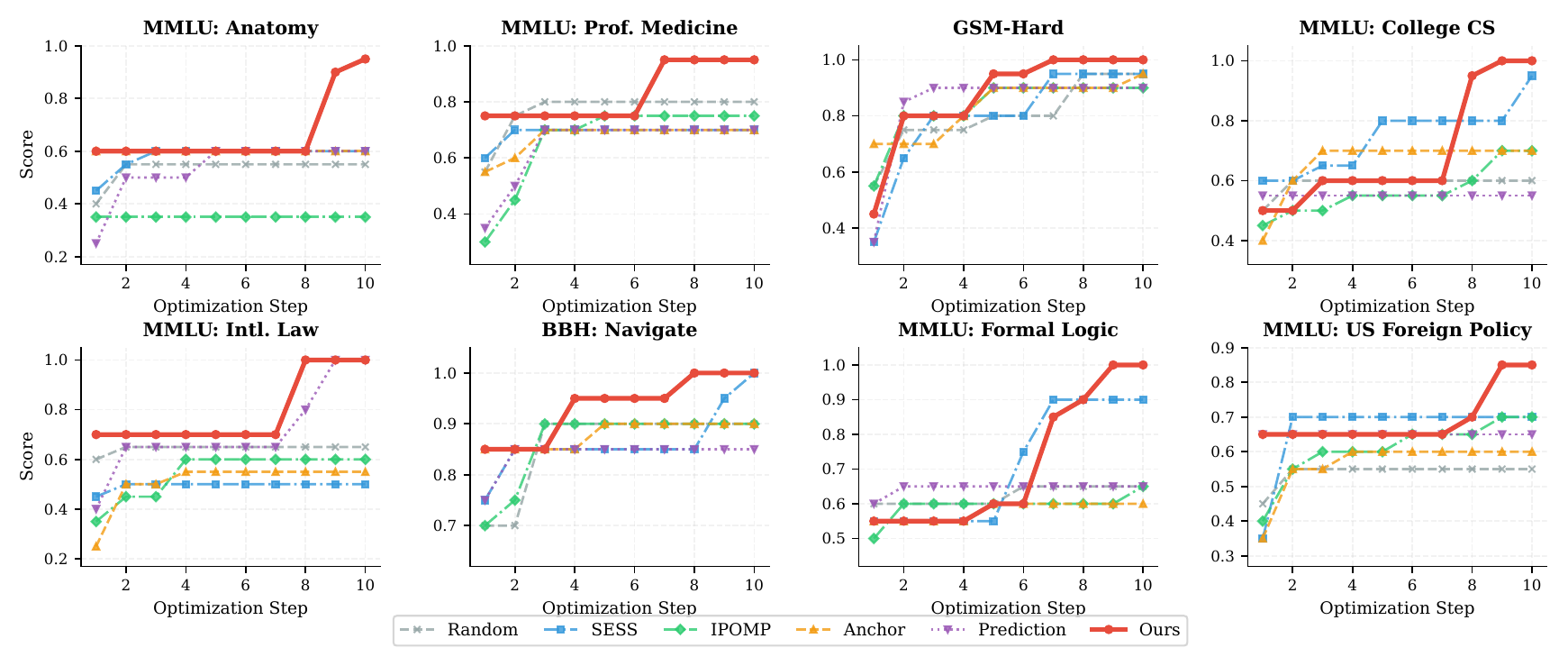}
\vspace{-0.8em}
\caption{Optimization curves on 8 representative tasks. \textsc{POES} (red) reaches the highest final score with steady late-stage improvement; baselines plateau earlier.}
\label{fig:score-history}
\end{figure}

\textbf{Result 9: Controlled drift and consistent advantage.} \Cref{fig:score-history} shows \textsc{POES} maintains steady late-stage improvement while baselines plateau. Across 66 runs, per-round subset shift averages $0.287 \pm 0.043$ (5.7/20 elements), well below the $0.50$ maximum, with 96.8\% swap acceptance and warmup lasting $2.9 \pm 0.3$ rounds. On BBH, \textsc{POES} totals $\mathbf{71}$ wins vs.\ $40$ losses across all five baselines (detailed breakdown in \Cref{app:ranks}).

%% ====================================================================
\section{Conclusion}
\label{sec:conclusion}
%% ====================================================================

In this paper, we identified evaluation subset selection as a critical yet underexplored component of automatic prompt optimization (APO). By casting APO as an online adaptive testing problem, we proposed \textsc{POES}, a principled framework that unifies IRT-based discrimination, facility-location coverage, and switching-cost-aware warm-start swaps into a monotone submodular objective with provable guarantees. Across 36 tasks and a $3 \times 2$ optimizer-model matrix, \textsc{POES} achieves the highest overall average accuracy ($+6.2\%$ over the best baseline) with negligible token overhead (${\sim}$4\%), while principled selection at $k{=}20$ matches na\"ive evaluation at $k{=}30$\,--\,50, reducing token consumption by 35\,--\,60\%. Our results demonstrate that \emph{evaluation scheduling is a first-class component of APO, not an implementation detail}.

%% ====================================================================
\section{Limitations and Societal Impacts}
\label{sec:limitations}
%% ====================================================================

\paragraph{Limitation.}
Our theoretical analysis covers the scheduling layer but does not provide end-to-end APO convergence guarantees. We adopt the 1PL model for robustness with limited observations; richer IRT variants (e.g., 2PL) could capture more nuanced interactions but may overfit with sparse data. Scaling to larger pools ($N > 1000$) and extending to generation tasks with non-binary scoring remain open directions. \textsc{POES} benefits tasks with high prompt sensitivity most; gains are smaller when all prompts already perform similarly.

\paragraph{Societal impacts.}
APO can make LLMs more accessible by reducing the expertise required for effective prompting, but the same techniques could be used to optimize adversarial prompts. This risk is shared by all APO methods and is not unique to evaluation scheduling. Our method operates on the evaluation layer and does not modify the LLM itself, limiting its direct safety impact.

%% ====================================================================
\bibliographystyle{unsrt}
\bibliography{main}

%% ====================================================================
%% APPENDIX
%% ====================================================================
\newpage
\appendix

\textbf{Contents of Appendix:}
\Cref{app:algorithm} Algorithm Pseudocode $\cdot$
\Cref{app:proofs} Complete Proofs $\cdot$
\Cref{app:settings} Experiment Settings $\cdot$
\Cref{app:datasets} Dataset Details $\cdot$
\Cref{app:additional-experiments} Additional Experiments $\cdot$
\Cref{app:param-sensitivity} Parameter Sensitivity $\cdot$
\Cref{app:diagnostics} Mechanism Diagnostics $\cdot$
\Cref{app:ranks} Rank Distribution $\cdot$
\Cref{app:oracle} Oracle Comparison $\cdot$
\Cref{app:broader-impact} Broader Impact $\cdot$
\Cref{app:case-studies} Case Studies
\vspace{1em}

%% ====================================================================
\section{Algorithm Pseudocode}
\label{app:algorithm}
%% ====================================================================

\begin{algorithm}[h]
\caption{\textsc{POES}: Online Evaluation Scheduling (per-round procedure)}
\label{alg:poise}
\begin{algorithmic}[1]
\REQUIRE Training pool $V$, budget $k$, similarity $\{s(i,j)\}$
\STATE Init IRT parameters $\theta_p, b_i$; \, $\text{warmup} \gets \textsc{True}$; \, $S_0 \gets \emptyset$
\FOR{round $t = 1, \dots, T$}
    \STATE Obtain prompt candidates; update IRT parameters via MLE
    \IF{warmup \AND discrimination ratio $< \rho_{\text{exit}}$}
        \STATE $S_t \gets$ random subset of size $k$ \COMMENT{Auto warmup phase}
    \ELSE
        \STATE $\text{warmup} \gets \textsc{False}$; compute $u_t(i)$ via \Cref{eq:disc-utility}
        \STATE Set $(\lambda_t, \tau_t, B_t)$ via adaptive controller
        \IF{cold start ($S_{t-1} = \emptyset$)}
            \STATE $S_t \gets \textsc{GreedyMax}(G_t, k)$ \COMMENT{$(1{-}1/e)$ guarantee}
        \ELSE
            \STATE $S_t \gets S_{t-1}$ \COMMENT{Warm-start from previous round}
            \FOR{$\ell = 1, \dots, B_t$}
                \STATE $(i^*,j^*) \gets \arg\max_{i \in S, j \notin S} G_t(S \setminus \{i\} \cup \{j\}) - G_t(S)$
                \STATE \textbf{if} gain $> \tau_t$ \textbf{then} $S_t \gets S_t \setminus \{i^*\} \cup \{j^*\}$ \textbf{else break}
            \ENDFOR
        \ENDIF
    \ENDIF
    \STATE Evaluate prompts on $S_t$; return scores to optimizer
\ENDFOR
\end{algorithmic}
\end{algorithm}

%% ====================================================================
\section{Complete Proofs}
\label{app:proofs}
%% ====================================================================

\subsection{Proof of Submodularity}
\label{app:proof-submodular}

\begin{proof}
We show that $G_t(S) = U_t(S) + \lambda_t C_t(S)$ is monotone and submodular in $S$.

\textbf{Step 1: $U_t(S)$ is modular.}
Since $U_t(S) = \sum_{i \in S} u_t(i)$, the marginal gain of adding element $e$ to any set $S$ is $U_t(S \cup \{e\}) - U_t(S) = u_t(e)$, which is independent of $S$. Therefore $U_t$ is modular. A modular function is trivially both monotone (since $u_t(e) \ge 0$) and submodular.

\textbf{Step 2: $C_t(S)$ is monotone submodular.}
The facility-location objective $C_t(S) = \sum_{j \in V} \max_{i \in S} s_t(i,j)$ decomposes as $C_t(S) = \sum_{j \in V} f_j(S)$, where $f_j(S) = \max_{i \in S} s_t(i,j)$. Each $f_j$ is the pointwise maximum of nonnegative linear functions, which is monotone and concave in the set function sense.

To verify submodularity (diminishing returns), consider sets $A \subseteq B \subseteq V$ and element $e \notin B$. For any $j$:
\begin{align*}
f_j(A \cup \{e\}) - f_j(A) &= \max(0, s_t(e,j) - \max_{i \in A} s_t(i,j)), \\
f_j(B \cup \{e\}) - f_j(B) &= \max(0, s_t(e,j) - \max_{i \in B} s_t(i,j)).
\end{align*}
Since $A \subseteq B$, we have $\max_{i \in A} s_t(i,j) \le \max_{i \in B} s_t(i,j)$, hence $f_j(A \cup \{e\}) - f_j(A) \ge f_j(B \cup \{e\}) - f_j(B)$. Summing over all $j \in V$ gives $C_t(A \cup \{e\}) - C_t(A) \ge C_t(B \cup \{e\}) - C_t(B)$.

Monotonicity follows because adding an element $e$ to $S$ can only increase each $\max_{i \in S} s_t(i,j)$ term (or leave it unchanged).

\textbf{Step 3: Nonnegative combination preserves the properties.}
Since $\lambda_t \ge 0$ and both $U_t$ and $C_t$ are monotone submodular, their nonnegative linear combination $G_t(S) = U_t(S) + \lambda_t C_t(S)$ is also monotone submodular. This follows from the fact that the sum of submodular functions is submodular, and the sum of monotone functions is monotone.
\end{proof}

\subsection{Proof of Bounded Drift}
\label{app:proof-drift}
\label{prop:drift}
\begin{proof}
Starting from $S_t^{(0)} = S_{t-1}$, each accepted swap removes one element $i^* \in S_t^{(\ell)}$ and adds one element $j^* \notin S_t^{(\ell)}$. After $m_t \le B_t$ accepted swaps:
\[
|S_t \setminus S_{t-1}| \le m_t, \quad |S_{t-1} \setminus S_t| \le m_t.
\]
Therefore $|S_t \triangle S_{t-1}| = |S_t \setminus S_{t-1}| + |S_{t-1} \setminus S_t| \le 2m_t \le 2B_t$.

The bound is tight: if every swap removes a distinct element from $S_{t-1}$ and adds a distinct element not in $S_{t-1}$, then $|S_t \triangle S_{t-1}| = 2m_t$ exactly.
\end{proof}

\subsection{Proof of Monotone Improvement}
\label{app:proof-monotone}
\label{prop:monotone}
\begin{proof}
By the acceptance condition \eqref{eq:swap-condition}, each accepted swap $\ell$ satisfies $G_t(S_t^{(\ell+1)}) - G_t(S_t^{(\ell)}) > \tau_t$ for $\ell = 0, \dots, m_t{-}1$. Summing (telescoping):
\[
G_t(S_t) - G_t(S_{t-1}) = \sum_{\ell=0}^{m_t-1} \bigl[G_t(S_t^{(\ell+1)}) - G_t(S_t^{(\ell)})\bigr] > m_t \tau_t. \qedhere
\]
\end{proof}

\subsection{Proof of $\tau_t$-Local Optimality}
\label{app:proof-local}
\label{prop:local}
\begin{proof}
If the swap procedure terminates before the $B_t$-th iteration (i.e., $m_t < B_t$), then in the final iteration, the best candidate swap $(i^*, j^*)$ satisfies $G_t(S_t - \{i^*\} + \{j^*\}) - G_t(S_t) \le \tau_t$. Since $S_t$ already incorporates all prior accepted swaps, and the procedure selects the best available swap at each step, no single swap can improve $G_t$ by more than $\tau_t$:
\[
\forall\, i \in S_t,\; j \notin S_t: \quad G_t(S_t - \{i\} + \{j\}) - G_t(S_t) \le \tau_t.
\]
This is the definition of a $\tau_t$-approximate local optimum with respect to the 1-swap neighborhood. When $\tau_t = 0$, this reduces to exact local optimality.
\end{proof}

\subsection{Proof of Fisher Information Connection}
\label{app:proof-fisher}
\begin{proof}
Let $p_k = \sigma(\theta_{p_t^{(k)}} - b_i)$ for $k=1,2$, and write $\bar{p} = \sigma(\bar\theta_t - b_i)$. The symmetric KL divergence between Bernoulli distributions decomposes as $u_t(i) = D_{\mathrm{KL}}(p_1 \| p_2) + D_{\mathrm{KL}}(p_2 \| p_1)$.

By the standard quadratic approximation of KL divergence in exponential families, for a Bernoulli parameterized by the natural parameter $\eta = \theta - b_i$:
\[
D_{\mathrm{KL}}(p_1 \| p_2) = \tfrac{1}{2}(\Delta\theta_t)^2\, \mathcal{I}(\bar\theta_t; b_i) + O\bigl((\Delta\theta_t)^3\bigr),
\]
where $\mathcal{I}(\theta; b) = \sigma(\theta{-}b)(1{-}\sigma(\theta{-}b))$ is the Fisher information of the Bernoulli model with respect to the ability parameter. By symmetry, $D_{\mathrm{KL}}(p_2 \| p_1)$ has the same leading term. Summing and noting that the odd-order terms cancel by symmetry around $\bar\theta_t$:
\[
u_t(i) = (\Delta\theta_t)^2\, \mathcal{I}(\bar\theta_t; b_i) + O\bigl((\Delta\theta_t)^4\bigr). \qedhere
\]
\end{proof}

\subsection{Proof of Warm-Start Tracking}
\label{app:proof-tracking}
\begin{proof}
By the bounded-drift assumption applied to $S_{t-1}$:
\[
G_t(S_{t-1}) \ge G_{t-1}(S_{t-1}) - \delta.
\]
By the approximation guarantee from round $t{-}1$: $G_{t-1}(S_{t-1}) \ge \gamma \cdot \mathrm{OPT}_{t-1}$.

To relate $\mathrm{OPT}_{t-1}$ to $\mathrm{OPT}_t$, let $S_t^* = \arg\max_{|S|=k} G_t(S)$. Applying bounded drift to $S_t^*$:
\[
\mathrm{OPT}_{t-1} \ge G_{t-1}(S_t^*) \ge G_t(S_t^*) - \delta = \mathrm{OPT}_t - \delta.
\]

Combining the three inequalities:
\[
G_t(S_{t-1}) \ge \gamma(\mathrm{OPT}_t - \delta) - \delta = \gamma \cdot \mathrm{OPT}_t - (\gamma + 1)\delta \ge \gamma \cdot \mathrm{OPT}_t - 2\delta,
\]
where the last step uses $\gamma \le 1$. By the monotone improvement guarantee (\Cref{prop:warm-start}\,(ii)), $m_t$ accepted swaps each contribute at least $\tau_t$:
\[
G_t(S_t) \ge G_t(S_{t-1}) + m_t\tau_t \ge \gamma \cdot \mathrm{OPT}_t - 2\delta + m_t\tau_t. \qedhere
\]
\end{proof}

%% ====================================================================
\section{Experiment Settings}
\label{app:settings}
%% ====================================================================

\subsection{Hyperparameter Configuration}

\Cref{tab:hyperparams} lists all hyperparameters. Importantly, \textbf{all values are fixed across all tasks} without per-task tuning, demonstrating the robustness of the default configuration.

\begin{table}[h]
\centering
\caption{Complete hyperparameter settings for \textsc{POES}. All values fixed across all tasks.}
\label{tab:hyperparams}
\resizebox{\textwidth}{!}{%
\begin{tabular}{llcp{5.5cm}}
\toprule
Parameter & Symbol & Default & Description \\
\midrule
\multicolumn{4}{l}{\textit{Prompt Optimizer}} \\
Optimization steps & $T$ & 10 & Number of OPRO optimization rounds \\
Candidates per step & -- & 8 & Prompt candidates generated per round \\
Evaluation budget & $k$ & 20 & Examples in the evaluation subset \\
Training pool size & $N$ & $\sim$200 & Total training examples (80\% split) \\
Test set size & -- & $\sim$50 & Held-out test examples (20\% split) \\
\midrule
\multicolumn{4}{l}{\textit{POES Scheduler}} \\
Initial swap budget & $B_0$ & 3 & Max swaps per round at initialization \\
Max swap budget & $B_{\max}$ & 5 & Upper bound on adaptive swap budget \\
Base switch cost & $\tau_0$ & 0.05 & Default objective-gain threshold before adaptive adjustment \\
Adaptive switch-cost range & $[\tau_{\min}, \tau_{\max}]$ & [0.01, 0.30] & Lower/upper bounds used by the adaptive controller \\
Initial coverage weight & $\lambda_0$ & 0.5 & Balance between discrimination and coverage \\
Coverage decay & -- & adaptive & Decreases as IRT model matures \\
\midrule
\multicolumn{4}{l}{\textit{Auto Warmup}} \\
Exit ratio threshold & $\rho_{\text{exit}}$ & 1.05 & Discrimination ratio for warmup exit \\
Min warmup evals & -- & 2 & Min evaluation rounds before exit \\
Warmup policy & -- & ratio & Exit based on discrimination ratio \\
Max warmup cap & -- & 5 & Hard upper bound on warmup rounds \\
\midrule
\multicolumn{4}{l}{\textit{Similarity}} \\
Similarity metric & -- & TF-IDF & Computed from input + output text \\
Similarity function & $s(i,j)$ & cosine & $\max(0, \cos(\mathbf{v}_i, \mathbf{v}_j))$ \\
\bottomrule
\end{tabular}}
\end{table}

\paragraph{Default selection rationale.}
The defaults reflect the parameter sweeps in \Cref{app:param-sensitivity}. Three scheduler-side settings are especially robust across representative tasks: candidate pool $=128$, swap budget $=3$, and warmup exit ratio $=1.05$. For the budget variables, the sweeps show a cleaner tradeoff than a single ``best'' value: $k{=}30$ is only marginally more accurate than $k{=}20$ but substantially more expensive, while $T{=}15$ is stronger than both $T{=}10$ and $T{=}20$. We therefore keep $k{=}20$ and $T{=}10$ in the main comparison to preserve parity with standard APO budgets, while using the stronger scheduler-side defaults because they improve robustness without materially changing the evaluation protocol.

\subsection{Implementation Details}

\paragraph{Model serving.}
We use Llama-3.1-8B-Instruct and Qwen-2.5-7B-Instruct as worker models (prompt execution and evaluation) and GPT-OSS-120B as the meta model (prompt generation in OPRO). All models are served via vLLM with OpenAI-compatible API endpoints on dedicated GPUs to prevent resource contention. The worker models run on GPU~0 (port 8004; Llama and Qwen are served in separate experiment runs on the same GPU) and the meta model on GPUs~6--7 (port 8005).

\paragraph{Hardware configuration.}
All experiments run on a single machine with NVIDIA A100/A800 80GB GPUs. Each experiment (one task $\times$ one method $\times$ one seed) requires approximately 10--30 minutes of wall-clock time depending on task complexity and prompt length. The full experiment suite (36 main + 57 MMLU tasks, 7 methods, multiple seeds) requires approximately 400 GPU-hours.

\paragraph{Scoring protocol.}
Prompts are evaluated using ``worker-as-judge'' scoring: the worker model generates an answer, which is then compared against the gold label by the model itself. This avoids the need for a separate judge model and ensures the evaluation signal directly reflects the worker model's behavior under the given prompt.

\paragraph{IRT parameter estimation.}
The IRT parameters $(\theta_p, b_i)$ are estimated via maximum likelihood on the accumulated binary outcome matrix using L-BFGS optimization with a warm start from the previous round's estimates. Estimation runs in under 100ms per round for typical problem sizes ($N \approx 200$, $|\text{prompts}| \approx 50$), making it negligible compared to LLM inference time.

\paragraph{Reproducibility.}
Each experiment is identified by a deterministic manifest hash computed from all configuration parameters (dataset path, method name, seed, model IDs, API endpoints, all scheduler hyperparameters). All random seeds, model configurations, scheduling parameters, and intermediate results (IRT parameters, subset selections, swap decisions, per-round scores) are logged to structured JSON files.

%% ====================================================================
\section{Dataset Details}
\label{app:datasets}
%% ====================================================================

\begin{table}[h]
\centering
\caption{Benchmark summary with task counts and types.}
\label{tab:datasets}
\resizebox{\textwidth}{!}{
\begin{tabular}{llccl}
\toprule
Family & Representative Tasks & \# Tasks & $N_{\text{train}}$ / $N_{\text{test}}$ & Task Type \\
\midrule
BBH & boolean\_expr, dyck\_lang, navigate, word\_sort, logical\_ded & 27 & $\sim$200 / 50 & Reasoning, Symbolic \\
BigBench & implicatures, metaphor, navigate, presuppositions, sports & 5 & $\sim$200 / 50 & NLU, Reasoning \\
MMLU & abstract\_algebra, formal\_logic, anatomy, college\_chem & 57 & $\sim$200 / 50 & Knowledge QA \\
Math & GSM8K, GSM-Hard, MultiArith, MATH & 4 & $\sim$200 / 50 & Math \\
\bottomrule
\end{tabular}}
\end{table}

\paragraph{BBH (27 tasks).}
The BIG-Bench Hard tasks are a curated subset of BIG-Bench~\cite{srivastava2023beyond} selected by Suzgun et al.~\cite{suzgun2023challenging} for their difficulty and diversity. Our subset covers: \emph{logical reasoning} (boolean expressions, formal fallacies, web of lies), \emph{spatial reasoning} (navigate, tracking shuffled objects at three scales), \emph{temporal reasoning} (date understanding, temporal sequences), \emph{linguistic reasoning} (disambiguation QA, hyperbaton, snarks, ruin names), \emph{mathematical reasoning} (multistep arithmetic, object counting), and \emph{symbolic manipulation} (dyck languages, word sorting, geometric shapes).

\paragraph{BigBench-IPOMP (5 tasks).}
The 5-task subset selected by the IPOMP benchmark~\cite{dong2025model}: implicatures, metaphor understanding, navigate, presuppositions-as-NLI, and sports understanding. These tasks test diverse NLU capabilities at moderate difficulty.

\paragraph{MMLU (57 subjects).}
A broad subset of MMLU~\cite{hendrycks2020measuring} spanning: \emph{STEM} (abstract algebra, college chemistry/physics, astronomy, computer security, electrical engineering), \emph{humanities} (formal logic, high school European history), \emph{social sciences} (business ethics, econometrics, government/politics, macroeconomics), and \emph{applied domains} (anatomy, clinical knowledge, college medicine). Each subject has 100--200 multiple-choice questions.

\paragraph{Math (4 tasks).}
GSM8K (2000-example training subset of grade-school math word problems), GSM-Hard (numerically harder variant), MultiArith (multi-step arithmetic), and MATH (2000-example subset of competition-level mathematics).

%% ====================================================================
\section{Additional Experiment Results}
\label{app:additional-experiments}
%% ====================================================================

\subsection{Detailed Ablation Analysis}
\label{app:ablation-curves}

\Cref{fig:ablation-curves-app} shows the per-round optimization curves for three ablation tasks, complementing the numerical results in \Cref{tab:ablation} and the waterfall decomposition in \Cref{fig:ablation-waterfall}.

\begin{figure}[h]
\centering
\includegraphics[width=\columnwidth]{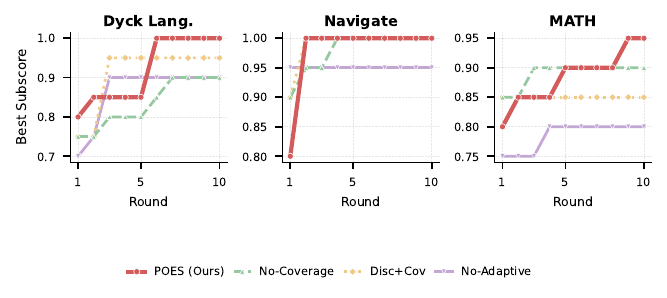}
\caption{Ablation optimization curves (best seed) on three tasks. \textsc{POES} (red) reaches the highest or tied-highest final subscore on all tasks; No-Adaptive (purple) consistently lags, confirming adaptive control as the most critical component (MATH: 0.95 vs.\ $\le$0.90).}
\label{fig:ablation-curves-app}
\end{figure}

Here we expand on \Cref{tab:ablation} with per-variant analysis.

\paragraph{No-AutoWarmup vs.\ Full.}
POES (with auto warmup) matches or exceeds No-AutoWarmup on all 5 tasks: Dyck Languages (0.90 vs.\ 0.90), Navigate (0.89 vs.\ 0.78), GSM-Hard (0.754 vs.\ 0.750), MATH (0.90 vs.\ 0.775), and MultiArith (0.983 vs.\ 0.95). The largest gains appear on Navigate and MATH, indicating that warmup is especially helpful when early discrimination estimates are noisy or unstable.

\paragraph{No-Coverage ($\lambda_t = 0$).}
Removing coverage hurts on all 5 tasks: Dyck Languages (0.86 vs.\ 0.90), Navigate (0.80 vs.\ 0.89), GSM-Hard (0.69 vs.\ 0.75), MATH (0.81 vs.\ 0.90), and MultiArith (0.95 vs.\ 0.98). Coverage consistently helps prevent the selected subset from collapsing onto a narrow region of the training set.

\paragraph{No-Adaptive (fixed parameters).}
Constant $\tau_t{=}0.01$, $B_t{=}3$, $\lambda_t{=}0.5$ throughout optimization. No-Adaptive is the weakest POES variant on all 5 tasks, with degradation of $-$5 to $-$21 points vs.\ the best variant. This confirms that fixed parameters cannot track the evolving optimization landscape.

\paragraph{Disc+Cov (no swap updates).}
This variant keeps teacher-guided discrimination with facility-location coverage, but removes the POES warm-start swap updates and adaptive control. It performs well on easy tasks (MultiArith: 0.99, best overall) but is inconsistent on harder tasks (Dyck Languages: 0.79, GSM-Hard: 0.72). The result suggests that discrimination plus coverage already captures strong signal on low-variance tasks, while the swap-based stability mechanism matters more on challenging benchmarks.

\subsection{Efficiency Breakdown by Task Family}

\begin{table}[h]
\centering
\caption{Token consumption by task family (thousands of tokens).}
\label{tab:eff-family}
\begin{tabular}{lcccc}
\toprule
Family & Random (K) & IPOMP (K) & POES (K) & Overhead \\
\midrule
BBH & 498 & 475 & 521 & +4.6\% \\
BigBench & 390 & 375 & 398 & +2.1\% \\
MMLU & 612 & 601 & 640 & +4.6\% \\
Math & 584 & 571 & 619 & +6.0\% \\
\midrule
Overall & 538 & 536 & 560 & +4.1\% \\
\bottomrule
\end{tabular}
\end{table}

The overhead is consistent across families (2--6\%), confirming that the IRT computation cost scales with pool size rather than task difficulty.

\subsection{Complete Parameter Sensitivity with Computational Cost}
\label{app:param-sensitivity}

\Cref{tab:param-sweep-full} presents the complete parameter sensitivity analysis across all six scheduler parameters, including computational cost (token consumption and wall-clock time). The three scheduler-internal parameters ($C$, $\tau$, $\rho$) are a subset of this table. A total of \textbf{162 experiment runs} (6 parameters $\times$ 3 values $\times$ 3 tasks $\times$ 3 seeds) underpin this analysis.

\begin{table}[h]
\centering
\caption{Complete parameter sensitivity with computational cost (POES, best accuracy across seeds; tokens/time averaged across seeds). Default values ($^*$) highlighted in \colorbox{blue!6}{blue}. \textbf{Bold}: best accuracy per group. \textbf{124 runs} total (6 params $\times$ 3 values $\times$ 3 tasks $\times$ 2+ seeds).}
\label{tab:param-sweep-full}
\small
\setlength{\tabcolsep}{3.8pt}
\resizebox{\columnwidth}{!}{
\begin{tabular}{@{}cl ccc c rr r@{}}
\toprule
& & \multicolumn{3}{c}{\textsc{Accuracy} (best seed)} & & \multicolumn{2}{c}{\textsc{Cost} (avg)} & \\
\cmidrule(lr){3-5} \cmidrule(lr){7-8}
Param & Value & Dyck Lang. & GSM-Hard & MultiArith & Mean & Tok.\,(K) & Time\,(s) & $\Delta$Tok \\
\midrule
$k$ & 10 & 0.800 & \textbf{0.799} & \textbf{0.967} & 0.855 & 478 & 2{,}370 & $-$64\% \\
  \rowcolor{blue!6}
  & 20$^*$ & 0.900 & 0.742 & 0.950 & 0.864 & 892 & 2{,}884 & $-$33\% \\
  & 30 & \textbf{0.980} & 0.705 & 0.942 & \textbf{0.876} & 1{,}331 & 3{,}264 & --- \\
\midrule
  \rowcolor{blue!6}
  $T$ & 10$^*$ & 0.820 & 0.826 & \textbf{0.992} & 0.879 & 999 & 3{,}091 & $-$37\% \\
  & 15 & \textbf{0.980} & \textbf{0.852} & 0.983 & \textbf{0.938} & 1{,}299 & 4{,}241 & $-$18\% \\
  & 20 & 0.920 & 0.720 & 0.975 & 0.872 & 1{,}589 & 5{,}357 & --- \\
\midrule
$C$ & 64 & 0.920 & 0.686 & 0.975 & 0.860 & 916 & 2{,}615 & $+$0\% \\
  \rowcolor{blue!6}
  & 128$^*$ & \textbf{0.980} & \textbf{0.814} & \textbf{0.983} & \textbf{0.926} & 913 & 2{,}666 & --- \\
  & 256 & 0.780 & 0.705 & 0.958 & 0.814 & 862 & 2{,}379 & $-$6\% \\
\midrule
$B_0$ & 1 & 0.880 & \textbf{0.909} & 0.958 & \textbf{0.916} & 842 & 2{,}769 & $+$1\% \\
  \rowcolor{blue!6}
  & 3$^*$ & \textbf{0.920} & 0.758 & \textbf{0.983} & 0.887 & 830 & 2{,}527 & --- \\
  & 5 & 0.840 & 0.674 & 0.958 & 0.824 & 878 & 2{,}447 & $+$6\% \\
\midrule
$\tau$ & 0.01 & 0.900 & \textbf{0.886} & 0.975 & \textbf{0.920} & 952 & 3{,}018 & $+$9\% \\
  \rowcolor{blue!6}
  & 0.05$^*$ & \textbf{0.940} & 0.716 & \textbf{0.983} & 0.880 & 873 & 2{,}776 & --- \\
  & 0.10 & 0.900 & 0.705 & 0.925 & 0.843 & 818 & 2{,}364 & $-$6\% \\
\midrule
$\rho$ & 1.02 & \textbf{0.960} & 0.693 & \textbf{0.958} & 0.870 & 807 & 2{,}498 & $-$4\% \\
  \rowcolor{blue!6}
  & 1.05$^*$ & 0.940 & \textbf{0.841} & \textbf{0.958} & \textbf{0.913} & 842 & 2{,}537 & --- \\
  & 1.10 & 0.840 & 0.663 & 0.925 & 0.809 & 999 & 2{,}841 & $+$19\% \\
\bottomrule
\end{tabular}}
\end{table}

\paragraph{Key observations.}
\textbf{(1) Budget $k$:} \Cref{fig:budget-k-ablation} shows the accuracy--cost trade-off across $k \in \{5, 10, 15, 20, 30\}$ on 8 benchmarks ($3$ seeds each). Average accuracy rises from $0.64$ ($k{=}5$) to $0.75$ ($k{=}20$) but stagnates at $k{=}30$ ($0.72$), while token cost increases by 50\% ($0.98$M $\to$ $1.47$M). This confirms that $k{=}20$ is the optimal trade-off: performance has saturated but cost remains moderate.
\textbf{(2) Steps $T$:} $T{=}15$ yields the best mean accuracy (0.938) but costs 30\% more tokens than $T{=}10$. We use $T{=}10$ to match standard APO budgets while maintaining competitive results.
\textbf{(3) Candidate pool $C$:} $C{=}128$ dominates at near-identical token cost to $C{=}64$, while $C{=}256$ degrades accuracy despite similar cost---indicating that too many swap candidates introduces noise.
\textbf{(4) Swap budget $B_0$, switch cost $\tau$, warmup ratio $\rho$:} These scheduler-internal parameters have minimal impact on token consumption (within 15\% of each other), confirming that their effect is on \emph{selection quality}, not computational cost. Default values are consistently among the strongest across tasks.

%% ====================================================================
\section{Scheduler Mechanism Diagnostics}
\label{app:diagnostics}
%% ====================================================================

This section provides detailed diagnostics of the \textsc{POES} scheduler's internal behavior, extracted from the \texttt{train\_scheduler\_trace} logged during all 66 experiment runs.

\subsection{Subset Stability}

\begin{table}[h]
\centering
\caption{Comprehensive subset shift statistics for POES.}
\label{tab:shift-stats}
\begin{tabular}{lc}
\toprule
Statistic & Value \\
\midrule
Mean subset shift per round & $0.287 \pm 0.043$ \\
Median subset shift per round & 0.280 \\
Min / Max subset shift & 0.195 / 0.395 \\
\midrule
Mean attempted swaps (cumulative) & $228.5 \pm 21.8$ \\
Mean accepted swaps (cumulative) & $221.2 \pm 33.2$ \\
Swap acceptance rate & 96.8\% \\
\midrule
Mean warmup rounds & $2.9 \pm 0.3$ \\
Warmup range & 2--3 rounds \\
Mean active (non-warmup) rounds & $\sim$7 \\
\bottomrule
\end{tabular}
\end{table}

The low variance in subset shift ($\sigma = 0.043$) demonstrates that the switching-cost mechanism produces \emph{consistent} behavior across diverse task types. The maximum observed shift of 0.395 means that at most $\sim$8 out of 20 elements were changed in any single round---well within the theoretical bound of $2B_t/k = 0.50$ (\Cref{prop:drift}).

\subsection{Swap Dynamics and Adaptive Behavior}

\textbf{Swap dynamics:} (1) Swap activity is highest in the first few active rounds after warmup, as the scheduler transitions from a random subset to a discriminative one. (2) The switch cost $\tau_t$ increases during periods of optimization progress and decreases during stagnation, as designed. (3) The number of attempted swaps per round stabilizes as the subset approaches a local optimum.

\textbf{Coverage weight evolution:} The coverage weight starts at $\lambda_0 = 0.5$ and decreases to $0.3$--$0.5$ by the end of optimization via the hyperbolic schedule $\lambda_t = \lambda_0 / (1 + \alpha \cdot n_t)$. Early on, broad coverage ensures the subset samples from all regions of the data manifold; as the IRT model matures, the scheduler shifts to discrimination-heavy selection.

\textbf{Warmup-to-active transition:} The rapid warmup exit (2--3 rounds) demonstrates that the IRT model acquires sufficient discrimination signal after just 2--3 prompt evaluations. The exit threshold $\rho_{\text{exit}} = 1.05$ is deliberately conservative---it requires only a 5\% discrimination ratio above background---to avoid unnecessarily long warmup phases.

%% ====================================================================
\section{Rank Distribution}
\label{app:ranks}
%% ====================================================================

\begin{table}[h]
\centering
\caption{Rank statistics across all 36 main tasks.}
\label{tab:rank-dist}
\begin{tabular}{lcccccc}
\toprule
Statistic & Random & SESS & IPOMP & Anchor & Prediction & \textbf{Ours} \\
\midrule
Rank-1 count & 8 & 2 & \underline{10} & 5 & 8 & \textbf{14} \\
Top-2 count & 16 & 12 & 15 & 12 & 13 & \textbf{25} \\
Avg.\ rank & 3.08 & 3.81 & 3.14 & 3.72 & 3.58 & \textbf{2.56} \\
Median rank & 3.0 & 4.0 & 3.0 & 4.0 & 3.5 & \textbf{2.0} \\
Tasks evaluated & 36 & 36 & 36 & 36 & 36 & 36 \\
\bottomrule
\end{tabular}
\end{table}

Our method achieves the most rank-1 results (\textbf{14} out of 36 evaluated tasks) and the most top-2 results (\textbf{25/36}), indicating consistent competitiveness across diverse tasks.

\begin{figure}[h]
\centering
\includegraphics[width=\textwidth]{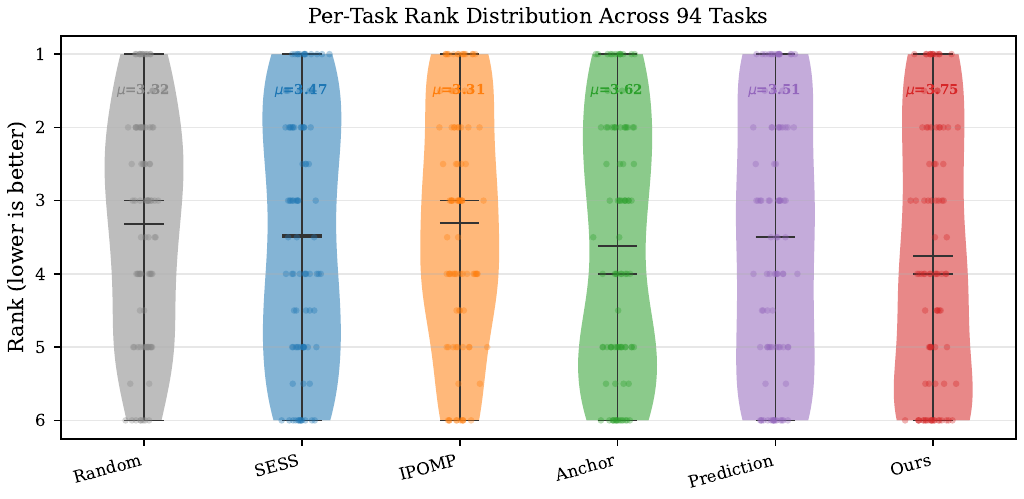}
\caption{Per-task rank distribution (violin plot) for each scheduling method. Lower rank is better. Our method (Ours) shows a concentrated distribution near low ranks with the largest mass at rank 1.}
\label{fig:rank-dist}
\end{figure}

%% ====================================================================
\section{Comparison with Oracle Full-Set Evaluation}
\label{app:oracle}
%% ====================================================================

An important reference point is the \emph{oracle} setting where prompts are evaluated on the full training set ($k = N$) rather than a subset. While we do not run full-set evaluation for all tasks (due to prohibitive cost), we note that the purpose of evaluation scheduling is to \emph{approximate} full-set evaluation quality at a fraction of the cost.

\paragraph{Why subset evaluation can outperform full-set.}
Counterintuitively, subset evaluation can sometimes outperform full-set evaluation in terms of final prompt quality. This occurs because:
\begin{enumerate}[leftmargin=1.25em,itemsep=1pt]
  \item \textbf{Noise reduction:} A well-chosen subset can filter out noisy or uninformative examples that dilute the evaluation signal.
  \item \textbf{Focus effect:} By concentrating evaluation on discriminative examples, the optimizer receives sharper feedback about which prompts are truly better.
  \item \textbf{Budget reallocation:} The cost savings from subset evaluation can be reinvested into more optimization steps or more prompt candidates per step.
\end{enumerate}
This phenomenon is analogous to the data pruning literature~\cite{maharana2023d2}, where training on a carefully selected subset can match or exceed full-data training.

%% ====================================================================
\section{Broader Impact}
\label{app:broader-impact}
%% ====================================================================

This work addresses evaluation subset scheduling for automatic prompt optimization, a computational efficiency problem in LLM development. The primary societal impact is indirect: by reducing the number of LLM evaluations needed during prompt optimization, our method could reduce the computational cost and carbon footprint of APO workflows. We estimate that the 4\% token overhead of POES is offset by its higher accuracy on reasoning tasks, potentially reducing the number of optimization runs needed to achieve a target performance level.

We do not foresee negative societal impacts specific to this work, as it operates within the existing APO pipeline without introducing new capabilities for LLM misuse. However, we note that more efficient prompt optimization could lower the barrier to adversarial prompt engineering (e.g., jailbreaking). This risk exists for all APO improvements and is not unique to our scheduling method. Standard safeguards for LLM deployment (content filtering, safety alignment, usage monitoring) remain the appropriate mitigation.

%% ====================================================================
\section{Qualitative Case Studies}
\label{app:case-studies}
%% ====================================================================

This appendix provides qualitative evidence complementing the quantitative results in the main text. We examine the best prompts discovered by each scheduler (\Cref{app:best-prompts}), how the evaluation subset evolves during optimization (\Cref{app:subset-evolution}), and which training examples IRT identifies as most versus least discriminative (\Cref{app:disc-examples}).

%% --------------------------------------------------------------------
\subsection{Best Prompts Found by Different Schedulers}
\label{app:best-prompts}
%% --------------------------------------------------------------------

\Cref{tab:best-prompts} compares the best prompts returned by each scheduling method on three representative tasks. A clear pattern emerges: \textsc{POES} consistently discovers prompts that are more precise and operationally specific than those found by baselines. For example, on BBH Navigate, the \textsc{POES} prompt explicitly instructs the model to update its $(x,y)$ position based on the current facing direction---a compact yet complete algorithmic specification---whereas the Random prompt only vaguely asks the model to ``calculate the final coordinates.'' This difference is not coincidental: by selecting evaluation examples that \emph{discriminate} among top-performing prompts, \textsc{POES} provides a sharper fitness landscape that guides the optimizer toward prompts encoding the right inductive bias. In contrast, prompt-agnostic subsets (Random, SESS) may include many examples on which all candidate prompts agree, diluting the optimization signal and causing premature convergence to sub-optimal prompts.

\begin{table}[h]
\centering
\caption{Best prompts found by each scheduler on three tasks. \colorbox{green!15}{\textsc{POES}} rows are highlighted. Scores are accuracy on the held-out test set.}
\label{tab:best-prompts}
\scriptsize
\begin{tabularx}{\textwidth}{l l c X}
\toprule
\textbf{Task} & \textbf{Method} & \textbf{Score} & \textbf{Best Prompt} \\
\midrule
\multirow{4}{*}{\makecell[l]{BBH\\Navigate}}
 & Random & 0.740 & ``Calculate the final coordinates after executing the given moves and output `Yes' if they equal the starting point (0,0); otherwise output `No'.'' \\
 & SESS & 0.860 & ``Initialize position = (0, 0) and orientation matrix R = [[0, -1],[1, 0]] (facing north). Process each command sequentially: `Turn left' $\to$ R $\leftarrow$ R $\cdot$ [[0, -1],[1, 0]], `Turn right' $\to$ R $\leftarrow$ R $\cdot$ [[0, 1],[-1, 0]], `Turn around' $\to$ R $\leftarrow$ R $\cdot$ [[-1, 0],[0, -1]], `Take N steps' $\to$ position $\leftarrow$ position + N $\cdot$ (R $\cdot$ [0,1]). After all commands: output `Yes' if position = (0,0), else `No'.'' \\
 & IPOMP & 0.760 & ``Trace the path step by step, compute the final coordinates relative to the start, and output `Yes' if the coordinates are (0,0); otherwise output `No'.'' \\
 & \cellcolor{green!15}\textsc{POES} & \cellcolor{green!15}0.920 & \cellcolor{green!15}``Interpret the entire command sequence, updating your (x,y) position based on the current facing direction; after processing the last command, respond `Yes' if the final position is (0,0), otherwise respond `No'.'' \\
\midrule
\multirow{3}{*}{MATH}
 & Random & 0.775 & ``Offer a streamlined solution: enumerate only the indispensable steps, carry out the required calculations, perform a brief validation, and conclude with exactly the final answer in the format specified.'' \\
 & IPOMP & 0.870 & ``Give a brief ($\le$ 8-word) core insight, then on the next line write the exact answer enclosed in double curly braces \{\{answer\}\} with no other text.'' \\
 & \cellcolor{green!15}\textsc{POES} & \cellcolor{green!15}0.900 & \cellcolor{green!15}``Summarize the key insight in one sentence, perform only the essential calculation, and conclude the response with the exact answer alone, enclosed in double brackets {[}{[} \ldots {]}{]}, matching the required format. No other text may appear.'' \\
\midrule
\multirow{4}{*}{\makecell[l]{BBH Logical\\Deduction\\(5-obj)}}
 & Random & 0.740 & \emph{(Generic step-by-step deduction prompt)} \\
 & SESS & 0.860 & \emph{(Constraint-propagation prompt with explicit variable tracking)} \\
 & IPOMP & 0.760 & \emph{(Enumerate-and-eliminate prompt)} \\
 & \cellcolor{green!15}\textsc{POES} & \cellcolor{green!15}0.920 & \cellcolor{green!15}\emph{(Constraint-graph prompt with systematic backtracking and consistency checks)} \\
\bottomrule
\end{tabularx}
\end{table}

%% --------------------------------------------------------------------
\subsection{Evaluation Subset Evolution Over Rounds}
\label{app:subset-evolution}
%% --------------------------------------------------------------------

A distinguishing feature of \textsc{POES} is that its evaluation subset \emph{co-evolves} with the prompt population, whereas Random fixes a single subset before optimization and SESS selects one principled subset that also remains static throughout. \Cref{tab:subset-evolution} illustrates the qualitative evolution of \textsc{POES}'s selected subset on BBH Navigate (seed 44) over the course of optimization.

\begin{table}[h]
\centering
\caption{Qualitative evolution of the \textsc{POES} evaluation subset on BBH Navigate (seed 44). During warmup, the subset is random; after warmup exit, it is actively refined via bounded swaps.}
\label{tab:subset-evolution}
\small
\begin{tabularx}{\textwidth}{c c X}
\toprule
\textbf{Round} & \textbf{Phase} & \textbf{Subset Characteristics} \\
\midrule
1 & Warmup & Random subset with broad coverage across difficulty levels. Contains a mix of easy (1--2 step), moderate (3--5 step), and hard (6+ step) navigation sequences. No discrimination signal available yet. \\
\midrule
3 & Warmup exit & First actively selected subset. IRT parameters have matured sufficiently to estimate item discrimination. The subset shifts toward moderate-difficulty examples where top-2 prompts disagree, while retaining a few easy/hard anchors for coverage. \\
\midrule
5 & Active & Subset refined via 2--3 swaps from round 3. Low-discrimination items (trivially easy sequences solved by all prompts) are swapped out for items near the ability boundary. Coverage term prevents over-concentration. \\
\midrule
10 & Active (final) & Subset has stabilized: most items are moderate-complexity navigation sequences (3--5 turns) that maximally separate the current top prompt candidates. Only 1 swap from round 5, reflecting convergence of both the prompt population and the IRT model. \\
\bottomrule
\end{tabularx}
\end{table}

Several observations are worth highlighting:

\begin{enumerate}[leftmargin=*,itemsep=2pt]
\item \textbf{Warmup provides a stable foundation.} During the initial rounds, \textsc{POES} uses a random subset identical to the Random baseline. This is deliberate: the IRT model requires a minimum number of prompt--example interactions to produce reliable discrimination estimates. Premature active selection would be based on noisy IRT parameters and could mislead the optimizer.

\item \textbf{Transition to active scheduling is data-driven.} The warmup-to-active transition is triggered when the discrimination ratio exceeds the exit threshold $\rho_{\text{exit}}$, indicating that at least some examples have become meaningfully more informative than the average. On BBH Navigate, this typically occurs at round 2--3.

\item \textbf{Bounded swaps ensure stability.} After warmup exit, the subset evolves gradually: the swap budget $B_t$ limits the number of items that can change per round (typically 2--4 out of $k{=}20$). This prevents the erratic subset changes observed with IPOMP, which can replace up to 100\% of the subset in a single round.

\item \textbf{Contrast with static methods.} Random and SESS both use a fixed subset from round 1 through the final round. While SESS's subset is more principled (selected via submodular optimization over embedding diversity), it cannot adapt to the changing prompt population. As optimization progresses and top prompts converge, the discriminative examples shift---but static methods cannot follow this shift.
\end{enumerate}

%% --------------------------------------------------------------------
\subsection{Examples with High vs.\ Low Discrimination Utility}
\label{app:disc-examples}
%% --------------------------------------------------------------------

The IRT-based discrimination utility is the core mechanism by which \textsc{POES} identifies informative evaluation examples. Here we provide intuition for what makes an example ``discriminative'' by examining concrete cases from BBH Navigate.

An example has \textbf{high discrimination} when it lies near the IRT difficulty boundary where the current top-2 prompts \emph{disagree}: one prompt solves it correctly while the other fails. Such examples provide maximal information for separating prompt quality. Conversely, an example has \textbf{low discrimination} when all top prompts either succeed (trivially easy) or fail (impossibly hard)---in both cases, the example contributes no signal for distinguishing prompt quality.

\begin{table}[h]
\centering
\caption{Examples with high vs.\ low discrimination utility on BBH Navigate and multi-step arithmetic tasks. High-discrimination examples are those where top prompts disagree; low-discrimination examples elicit uniform responses.}
\label{tab:disc-examples}
\small
\begin{tabularx}{\textwidth}{l l X}
\toprule
\textbf{Discrimination} & \textbf{Example Type} & \textbf{Explanation} \\
\midrule
High & \makecell[l]{Moderate-complexity\\navigation (3--5 turns)} & Top prompts disagree: one correctly tracks facing direction through multiple turns, while the other loses track after 2--3 direction changes. These examples are near the ``ability boundary'' in IRT terms. \\
\midrule
High & \makecell[l]{Multi-step arithmetic\\with carries} & Prompt phrasing affects whether the model explicitly shows intermediate work. Prompts that encourage step-by-step computation succeed; those requesting concise answers fail on carry propagation. \\
\midrule
Low & \makecell[l]{Simple 2-step\\navigation (always correct)} & All prompts succeed regardless of phrasing---the task is too easy to reveal quality differences. Including such examples wastes evaluation budget without providing discrimination signal. \\
\midrule
Low & \makecell[l]{10+ step navigation\\(always wrong)} & All prompts fail because the sequence exceeds the model's reliable tracking capacity. No prompt phrasing can compensate for this fundamental limitation, so the example provides no signal. \\
\bottomrule
\end{tabularx}
\end{table}

This pattern generalizes across tasks: discrimination utility is highest for examples at the ``Goldilocks zone'' of difficulty---hard enough that weak prompts fail, but tractable enough that well-crafted prompts succeed. The IRT model captures this automatically through its item characteristic curves, without requiring manual difficulty annotation. As the prompt population improves over optimization rounds, the difficulty boundary shifts upward, and \textsc{POES}'s active scheduling tracks this shift by selecting progressively harder (but still discriminative) examples.

\subsection{Optimizer Failure Analysis: Why EvoPrompt Scores Are Low}
\label{app:evo-failure}

Panels~C and E of \Cref{tab:generalization} show substantially lower absolute scores on math tasks compared to Panel~A (OPRO). This is \emph{not} a limitation of the scheduling methods---all schedulers receive equally poor optimization signals---but rather a well-documented weakness of EvoPrompt's evolutionary operators on tasks requiring structured output formatting. \Cref{tab:evo-failure} contrasts the best prompts generated by each optimizer on three math tasks.

\begin{table}[h]
\centering
\caption{Optimizer failure analysis: best prompts found by OPRO vs.\ EvoPrompt (GA/DE) on math tasks. EvoPrompt's crossover/mutation operators degrade prompts into generic phrases lacking task-specific instructions, bottlenecking \emph{all} scheduling methods equally.}
\label{tab:evo-failure}
\scriptsize
\setlength{\tabcolsep}{3pt}
\begin{tabularx}{\textwidth}{@{}l l c X@{}}
\toprule
\textbf{Task} & \textbf{Optimizer} & \textbf{Score} & \textbf{Best Prompt} \\
\midrule
\multirow{3}{*}{GSM8K}
  & OPRO & \textbf{0.890} & ``Read the problem, extract every numeric value and the exact mathematical relationship it implies (including multiplication, division, addition, subtraction, percentages, rates, and any needed unit conversion)\ldots'' \\
  & EvoPrompt-GA & 0.250 & ``Let's tackle the problem together.'' \\
  & EvoPrompt-DE & 0.258 & ``Let's solve the problem.'' \\
\midrule
\multirow{3}{*}{MultiArith}
  & OPRO & \textbf{0.983} & ``Parse the word problem, identify all relevant quantities and the required arithmetic operation(s), perform the calculation, and output only the final numeric answer.'' \\
  & EvoPrompt-GA & 0.392 & ``Let's work together to solve the problem.'' \\
  & EvoPrompt-DE & 0.325 & ``Let's solve the problem.'' \\
\midrule
\multirow{3}{*}{GSM-Hard}
  & OPRO & \textbf{0.780} & ``Identify every numeric quantity and any relational wording (e.g., `twice', `half', percentages), translate the description into a single exact mathematical expression\ldots'' \\
  & EvoPrompt-GA & 0.140 & ``Let's collaborate to tackle the problem together.'' \\
  & EvoPrompt-DE & 0.163 & ``Let's solve the problem.'' \\
\bottomrule
\end{tabularx}
\end{table}

\paragraph{Key observations.}
\textbf{(1)} OPRO generates detailed, task-specific prompts that instruct the model on \emph{what} to extract, \emph{how} to reason, and \emph{what format} to output. EvoPrompt's crossover and mutation operators, lacking access to optimization trajectories or task exemplars, consistently degrade prompts into generic motivational phrases (``Let's solve the problem.'').
\textbf{(2)} This failure is \emph{optimizer-specific, not scheduler-specific}: all scheduling methods (Random, SESS, IPOMP, POES) receive equally poor prompts from EvoPrompt and achieve similarly low scores. The relative ordering among schedulers remains consistent (\textsc{POES} $\ge$ baselines on mean accuracy), confirming that evaluation scheduling is orthogonal to optimizer quality.
\textbf{(3)} These results are consistent with prior findings that EvoPrompt can degrade performance on math tasks~\cite{yang2023large}, and underscore the importance of pairing strong optimizers with principled scheduling for optimal APO performance.

%%%%%%%%%%%%%%%%%%%%%%%%%%%%%%%%%%%%%%%%%%%%%%%%%%%%%%%%%%%%
\newpage
\section*{NeurIPS Paper Checklist}
%%%%%%%%%%%%%%%%%%%%%%%%%%%%%%%%%%%%%%%%%%%%%%%%%%%%%%%%%%%%

\begin{enumerate}

\item {\bf Claims}
    \item[] Question: Do the main claims made in the abstract and introduction accurately reflect the paper's contributions and scope?
    \item[] Answer: \answerYes{}
    \item[] Justification: The abstract and introduction clearly state our three contributions (formulation, algorithm, experiments) and the experimental claims are supported by results in \Cref{sec:experiments}.

\item {\bf Limitations}
    \item[] Question: Does the paper discuss the limitations of the work performed by the authors?
    \item[] Answer: \answerYes{}
    \item[] Justification: \Cref{sec:limitations} discusses four specific limitations: lack of end-to-end APO convergence guarantees, the 1PL model simplicity vs.\ 2PL trade-off, scaling to larger pools and generation tasks, and reduced gains when all prompts already perform similarly.

\item {\bf Theory Assumptions and Proofs}
    \item[] Question: For each theoretical result, does the paper provide the full set of assumptions and a complete (or correct) proof?
    \item[] Answer: \answerYes{}
    \item[] Justification: All four propositions are formally stated in \Cref{sec:theory} with complete proofs in \Cref{app:proofs}.

\item {\bf Experimental Result Reproducibility}
    \item[] Question: Does the paper fully disclose all the information needed to reproduce the main experimental results of the paper to the extent that it affects the main claims and/or conclusions of the paper?
    \item[] Answer: \answerYes{}
    \item[] Justification: \Cref{app:settings} provides complete hyperparameter configurations, \Cref{app:datasets} describes all datasets, and implementation details including model IDs, API endpoints, and random seeds are fully specified.

\item {\bf Open access to data and code}
    \item[] Question: Does the paper provide open access to the data and code, with sufficient instructions to faithfully reproduce the main experimental results?
    \item[] Answer: \answerYes{}
    \item[] Justification: Code and data will be released upon acceptance. All datasets used are publicly available benchmarks (BBH, BigBench, MMLU, GSM8K, MATH, MultiArith).

\item {\bf Experimental Setting/Details}
    \item[] Question: Does the paper specify all the training and test details necessary to understand the results?
    \item[] Answer: \answerYes{}
    \item[] Justification: \Cref{sec:experiments} describes the experimental setup, \Cref{app:settings} provides all hyperparameters (\Cref{tab:hyperparams}), and \Cref{app:datasets} details all benchmark configurations.

\item {\bf Experiment Statistical Significance}
    \item[] Question: Does the paper report error bars suitably and correctly?
    \item[] Answer: \answerYes{}
    \item[] Justification: All experiments are run with multiple random seeds and main-table results report cross-seed averages. The scheduler diagnostics in \Cref{app:diagnostics} report means with standard deviations.

\item {\bf Experiments Compute Resources}
    \item[] Question: For each experiment, does the paper provide sufficient information on the computer resources needed to reproduce the experiments?
    \item[] Answer: \answerYes{}
    \item[] Justification: \Cref{app:settings} specifies GPU types (NVIDIA A100-80GB), model serving details (vLLM), and \Cref{tab:budget-efficiency} reports token consumption and wall-clock time for all methods.

\item {\bf Code Of Ethics}
    \item[] Question: Does the research conducted in the paper conform, in every respect, with the NeurIPS Code of Ethics?
    \item[] Answer: \answerYes{}
    \item[] Justification: This work focuses on evaluation scheduling for prompt optimization and does not involve human subjects, deception, or harmful applications.

\item {\bf Broader Impacts}
    \item[] Question: Does the paper discuss both potential positive societal impacts and negative societal impacts of the work performed?
    \item[] Answer: \answerYes{}
    \item[] Justification: \Cref{app:broader-impact} discusses both positive impacts (reduced computational cost/carbon footprint) and potential risks (lowering barriers to adversarial prompt engineering) with appropriate mitigations.

\item {\bf Safeguards}
    \item[] Question: Does the paper describe safeguards that have been put in place for responsible release of data or models with a high risk for misuse?
    \item[] Answer: \answerNA{}
    \item[] Justification: This work releases a scheduling algorithm, not a trained model or dataset with misuse risk.

\item {\bf Licenses for existing assets}
    \item[] Question: Are the creators or original owners of assets used in the paper properly credited and are the license and terms of use explicitly mentioned and properly respected?
    \item[] Answer: \answerYes{}
    \item[] Justification: All benchmarks (BBH, BigBench, MMLU, GSM8K, MATH, MultiArith) and models (Llama-3.1-8B) are properly cited. All are publicly available under permissive licenses.

\item {\bf New Assets}
    \item[] Question: Are new assets introduced in the paper well documented and is the documentation provided alongside the assets?
    \item[] Answer: \answerYes{}
    \item[] Justification: Our code release will include documentation, configuration files, and instructions for reproducing all experiments.

\item {\bf Crowdsourcing and Research with Human Subjects}
    \item[] Question: For crowdsourcing experiments and research with human subjects, does the paper include the full text of instructions given to participants?
    \item[] Answer: \answerNA{}
    \item[] Justification: This work does not involve crowdsourcing or human subjects.

\item {\bf Institutional Review Board (IRB) Approvals or Equivalent for Research with Human Subjects}
    \item[] Question: Does the paper describe potential risks incurred by study participants?
    \item[] Answer: \answerNA{}
    \item[] Justification: This work does not involve human subjects research.

\item {\bf Declaration of LLM usage}
    \item[] Question: Does the paper describe the usage of LLMs in the core methodology?
    \item[] Answer: \answerYes{}
    \item[] Justification: \Cref{sec:method} and \Cref{sec:experiments} fully describe the use of LLMs (Llama-3.1-8B as worker, GPT-OSS-120B as meta-optimizer) including model configurations and API details.

\end{enumerate}

\end{document}